\documentclass[11pt,a4paper]{article}

\usepackage[T1]{fontenc}
\usepackage{graphicx}
\usepackage{mathtools}
\usepackage{amssymb}
\usepackage{amsthm}
\usepackage{dsfont}
\usepackage{amsmath}
\usepackage{thmtools}
\usepackage{xcolor}
\usepackage{nameref}
\usepackage{babel}
\usepackage[authoryear,round]{natbib}
\usepackage{ifthen}
\usepackage{pgfplots}
\pgfplotsset{compat=1.18}
\usepackage{subcaption}

\definecolor{rgbblue}{RGB}{68,119,170}
\definecolor{rgbgreen}{RGB}{34,136,51}
\definecolor{rgbviolet}{RGB}{170,51,119}

\usepackage[
    colorlinks=true,
    linkcolor=blue,
    citecolor=blue,
    urlcolor=blue
]{hyperref}
\usepackage{cleveref}

\usepackage{indentfirst}
\newboolean{toggleappendix}

\newcommand{\storeproof}[2]{
  \ifthenelse{\boolean{toggleappendix}}%
    {\expandafter\gdef\csname proof-#1\endcsname{#2}}%
    {\begin{proof}#2\end{proof}}
}

\newcommand{\showproof}[1]{%
  \ifthenelse{\boolean{toggleappendix}}%
    {\begin{proof}\csname proof-#1\endcsname\end{proof}}%
    {}
}

\usepackage[left=2.5cm, right=2.5cm, top=2.5cm, bottom=2.5cm]{geometry}

\newtheorem{theorem}{Theorem}
\numberwithin{theorem}{section}

\newtheorem{corollary}[theorem]{Corollary}
\newtheorem{proposition}[theorem]{Proposition}

\newtheorem{definition}[theorem]{Definition}
\theoremstyle{remark}

\newtheorem{assumption}[theorem]{Assumption}

\allowdisplaybreaks

\newcommand{\E}{\mathbb E}

\newcommand{\M}{\mathcal M}
\newcommand{\U}{\mathcal U}
\newcommand{\Yres}{\mathcal Y}
\newcommand{\Ptwo}{\mathcal P_2}
\newcommand{\Risk}{\mathcal R}

\newcommand{\dd}{\,\mathrm d}

\DeclareMathOperator{\Law}{Law}
\DeclareMathOperator{\Tr}{Tr}

\DeclareMathOperator*{\argmin}{arg\,min}
\DeclarePairedDelimiter{\norm}{\lVert}{\rVert}
\DeclarePairedDelimiter{\ip}{\langle}{\rangle}
\usepackage{xcolor}

\allowdisplaybreaks[2]

\title{When do explicit penalties improve flow matching?}

\author{Lea Kunkel, Alberto González-Sanz and Sophie Langer}
\date{Oktober 2026}

\begin{document}
\maketitle

\begin{abstract}
Flow matching can learn structure directly from data, raising the question of when additional constraints improve generation. We study penalties that encode such constraints, with \emph{physics-informed flow matching} as a motivating example. Satisfying a constraint more closely need not make the generated distribution more accurate. We therefore compare penalised and unpenalised generative modelling through their expected squared $2$-Wasserstein error. Our analysis covers both vector-valued and functional data in a Hilbert-space setting. We first derive geometric criteria that account for constraint violations and the distribution of samples after projection onto the constraint set. We then solve the statistical comparison explicitly for specific velocity-field classes and constraints via flow matching. In this model, penalisation improves accuracy exactly when the reduction in estimation variance outweighs the increase in squared bias. The resulting formula shows how this trade-off depends on sample size, the model class, the noise schedule, and the penalty weight.
\end{abstract}

\section{Introduction}\label{sec:intro}
In scientific applications, obtaining data can require costly numerical simulations. In astrophysics, for example, the cost of high-fidelity cosmological simulations can limit the available training data \citep{zeghal2025bridging}. Generative models offer a promising approach by approximating the distribution of simulation outputs and provide a cheaper way to produce additional samples. Flow matching does this by learning a velocity field that transports a simple reference distribution towards the data distribution \citep{lipman2023flow,liu2023flow, albergo2022}. It has shown strong performance in tasks ranging from high-resolution image synthesis \citep{esser2024scaling} to speech generation \citep{chen2024f5tts}. With limited training data, however, the learned distribution may still differ substantially from the target.

Additional knowledge about the data, such as physical laws, algebraic relations between variables, or a known underlying subspace, can be incorporated through penalties on constraint violations. In \emph{physics-informed flow matching}, physical residuals or conservation errors enter the training objective \citep{baldan2026physics,debeire2026physics}. Numerical experiments with related physics-informed diffusion models also suggest that such penalties can reduce over-fitting \citep{bastek2025physics}.

The statistical benefit of such penalties is not guaranteed. In the edge case of a hard constraint, the admissible samples are restricted, but the constraint generally does not determine how probability is distributed among them. Satisfying a constraint more closely can therefore come at the cost of accurately reproducing the distribution within the constraint set.  For soft constraints, where violations are weighted against improvements elsewhere, a similar but more subtle setting occurs. Intrinsic adaption mechanisms compete with the external influences of the penalty. This raises the question we aim to answer: \emph{when does adding a constraint penalty improve the distribution estimation ability of flow matching?}

\paragraph{Problem formulation.}

We consider independent observations $Y_1,\ldots,Y_n$ from an unknown probability distribution $\nu$ on a separable Hilbert space $\mathcal H$. A known continuous residual map $R\colon\mathcal H\to\Yres$, with $\Yres$ another Hilbert space, defines the constraint set $\mathcal M=\{y:R(y)=0\}$, and we assume $\nu(\mathcal M)=1$. We fit a velocity field in a class $\mathcal V$ via the penalised flow-matching objective:
\begin{equation}\label{eq:intro-estimator}
 \widehat u_{n,\lambda}\in\argmin_{u\in\mathcal V}
 \Bigl\{\underbrace{F_n(u)}_{\text{Flow matching loss}}+\lambda \underbrace{\mathbb{E}_{X \sim P^u}\bigl[\|R(X)\|^2\bigr]}_{\text{Constraint penalty}}\Bigr\},\qquad \lambda\geq0.
\end{equation}
Here $P^u$ is the distribution resulting from using the vector field $u$ in the flow matching dynamics. The precise definitions are given in \Cref{sec:fm}, the concrete objective is defined in \eqref{eq:estimator}. Setting $\lambda=0$ gives the \emph{unpenalised flow-matching estimator}. 
 We measure accuracy by
\begin{equation}\label{eq:intro-risk}
 \Risk_n(\lambda)
 =\E_{ Y^{(n)}}\big[W_2^2(P^{\widehat u_{n,\lambda}},\nu)\big],
\end{equation}
where $Y^{(n)}=(Y_1,\ldots,Y_n)$ and $W_2$ is the $2$-Wasserstein distance, defined in \Cref{sec:geometry}. The expectation averages over the training sample. We ask when $\Risk_n(\lambda)<\Risk_n(0)$ for $\lambda>0$, comparing the estimators with the same observations, velocity-field class, training and sampling setup.

\paragraph{Our contributions.}
Building on functional flow matching (\cite{pmlr-v238-kerrigan24a, zhang2026flow}), we analyse when constraint penalties improve generation in separable Hilbert spaces. This setting accommodates both vector-valued and functional observations. We develop the comparison from two complementary perspectives.
\begin{enumerate}
\item \textbf{Geometric criteria for improved generation.} We separate distributional error into the distance of generated samples from the constraint set and the discrepancy between their projected distribution and the target. For affine constraint sets, this decomposition is exact (\Cref{cor:flat}). It shows that moving samples closer to the constraint set improves generation only when this gain outweighs any deterioration in the projected distribution (\Cref{cor:geocomparison}). For curved sets, we control the additional remainder under assumptions on curvature and proximity to the set (\Cref{ass:tubular} and \Cref{thm:geometric}). These criteria apply independently of how the distributions were learned (\Cref{cor:geocomparison}). 
\item \textbf{Exact statistical comparison in an affine model.}
For affine residuals and a restricted affine velocity-field class, we derive the penalised empirical minimiser and its exact change in expected $2$-Wasserstein risk. The generated distribution has a fixed shape and a learned location. Within this model, penalisation improves accuracy precisely when the reduction in estimation variance exceeds the increase in squared bias (\Cref{thm:linear_linear_exact}). The comparison makes the dependence on sample size, the hypothesis class, the noise schedule, and the penalty weight explicit. In particular, for a fixed penalty and noise schedule, any strictly positive increase in squared bias eventually outweighs the variance benefit as the sample size increases. \Cref{cor:linear_comparison} extends the analysis to arbitrary residuals. The affine model connects the two perspectives: the geometric and statistical comparisons yield the same change in distributional error (\Cref{prop:connection}).
\item \textbf{Extension to fixed feature models.} As a first step towards network based velocity field estimators, we extent our analysis to fixed feature models. \Cref{thm:last_layer_improvement} shows that penalisation always improves distribution estimation for sufficiently small $\lambda> 0$. 
\end{enumerate}

\paragraph{Related literature.} 
Several recent works establish statistical guarantees or closely related results for unpenalised flow matching, corresponding to $\lambda=0$ in \eqref{eq:intro-estimator}. \cite{gao2024} establish a suboptimal convergence rate of the distributional estimator with respect to the $2$-Wasserstein distance. \cite{benton2024error} provide an analysis of flow matching excluding the approximation error of the velocity field. \cite{gong2025} investigate approximation-theoretic properties of ReLU neural networks for representing vector fields associated with higher-order trajectories. \cite{fukumizu2024, zhou2025error} partly improve previous results, but their proofs exhibit difficulties with the Lipschitz constant of the intrinsic vector field, which grows with the sample size and directly affects the resulting convergence rate. 
\cite{stephanovitch2026lipschitz,kunkel2025lipschitz} investigate this specific pitfall in flow matching analysis. \cite{kunkel2025minimax} obtain minimax convergence rates in an over-parametrised setting. \cite{mena2025statistical} study further properties, including asymptotic behaviour and the connection to optimal transport maps.

For intrinsically low-dimensional target distributions, \citet{chakraborty2026generalization} obtain high-probability bounds governed by the Wasserstein dimension, using stochastic sampling after flow-matching training. Together with \cite{kumar2026manifold}, these results illustrate how flow-matching estimators can exploit low-dimensional data structure without an explicit constraint penalty.

Known structure is incorporated explicitly in physics-informed neural networks (PINNs), which penalise violations of governing equations during training \citep{raissi2019physics}. \citet{mishra2023estimates} use PDE stability to bound solution error in terms of training residuals and quadrature error. \citet{doumeche2025convergence} establish consistency of PINNs under additional ridge regularisation and strong convergence for linear PDE systems with an additional Sobolev penalty. More directly related to our statistical question, \citet{doumeche2024kernel} study regression with a penalty based on a linear differential operator and show that physical information can improve convergence rates in specific settings. These results concern the estimation of a solution or regression function. In generative modelling, the additional task is to recover how probability is distributed among admissible samples.

Alternative approaches to incorporate constraint penalties in flow matching include fine-tuning with weak-form PDE residuals \citep{tauberschmidt2026physics} or corrections during sampling \citep{utkarsh2025physics} (see \Cref{sec:discussion} for more details). As mentioned above, our analysis concerns penalties added to the flow matching training objective.

\paragraph{Outline.}
\Cref{sec:fm} introduces flow matching for Hilbert-valued data, the end-point penalty and the corresponding empirical risk minimiser. \Cref{sec:geometry} characterises improvement through the penalty based on geometric criteria. \Cref{sec:estimator_based_analysis} and \Cref{sec:last_layer} investigate the impact of penalisation for affine and fixed feature models, respectively. The proofs of all results are deferred to the appendix.

\paragraph{Notation.}
All Hilbert spaces are real and separable. For a Hilbert space $\mathcal{H}$, we write $\langle \cdot , \cdot\rangle_\mathcal{H}$ and $\norm{\cdot}_\mathcal{H}$ for its inner product and norm,  respectively. When clear from the context, we abbreviate it to $\ip{\cdot,\cdot}$ and $\norm{\cdot}$. Their given inner products determine adjoints, covariance operators, and the Wasserstein cost. For a bounded linear operator $A$ between two Hilbert spaces, we denote the adjoint operator by $A^*$. Further we write $\operatorname{ran}(A)=\{A\theta:\theta\in\mathcal E\}$ for its range
and $\overline{\operatorname{ran}(A)}$ for its closure in the norm
of $\mathcal H$. The corresponding squared
Hilbert--Schmidt norm is
$
\|A\|^2_{\mathrm{HS}}
\coloneqq
\sum_j\|Ae_j\|_{\mathcal{H}}^{2},
$
where $(e_j)$ is an orthonormal basis of $\mathcal H$.
We denote by $\Ptwo(\mathcal H)$ the space of tight probability measures on $\mathcal H$ with finite second moments and by $\Law(X)$ the law of a random variable $X$. For a Borel measurable $T$ and a probability measure  $\rho$, we denote by  $T_\#\rho$ the push-forward measure of $\rho$ by $T$, defined by $T_\#\rho=\Law(T(Y)) $  for $Y\sim \rho$. 
For $a,b\in\mathcal H$, $(a\otimes b)h=\ip{h,b}a$. Expectations of Hilbert-valued random variables are Bochner expectations. For $X\sim\rho\in\Ptwo(\mathcal H)$, we write
\[
 m_\rho=\E X,\qquad
 \Sigma_\rho=\E[(X-m_\rho)\otimes(X-m_\rho)].
\]
All consequential properties needed for this work can be found in \citet[Section 7.2]{Eubank2015-lg}. For a Borel family of probability measures $(\mu_t)_{t\in[0,1]}$ and a finite measure $\eta$ on $[0,1]$, $L^2(\dd\eta(t)\mu_t;\mathcal H)$ denotes the space of jointly Borel fields $u:[0,1]\times\mathcal H\to\mathcal H$ satisfying
\[
 \int_0^1\int_{\mathcal H}\norm{u_t(x)}^2 \dd  \mu_t(x)\dd \eta( t)<\infty,
\]
with fields identified up to equality almost everywhere. We use $L^2(dt\,\mu_t;\mathcal H)$ when the time measure is the Lebesgue measure.

\section{Penalised flow matching in Hilbert spaces}\label{sec:fm}

We first recall flow matching for Hilbert-valued data, then introduce the constraint penalty and define the corresponding empirical risk minimiser. Functional flow matching was introduced by \citet{pmlr-v238-kerrigan24a} as an extension of the finite-dimensional framework of \citet{lipman2023flow} to infinite-dimensional spaces. Here, we use the conditional-expectation formulation of \citet{zhang2026flow}.

\paragraph{Unpenalised flow matching in Hilbert spaces.}

Flow matching aims to generate samples from an unknown distribution $\nu$ by learning a flow mapping that transports a known, tractable reference distribution $\mu_0$ to $\nu$. Let $Y \sim \nu \in \mathcal{P}_2(\mathcal{H})$ be a target sample and $Z\sim\mu_0\in\Ptwo(\mathcal{H})$ be an independent reference noise with $\E[Z]=0$. We define the paths, its velocity and its marginal law by
\begin{equation}\label{eq:path}
 X_t=tY+\sigma_tZ,\qquad
 V_t=Y+\sigma_t^{\prime}Z,\qquad
 \mu_t=\Law (X_t), \qquad t \in [0,1],
\end{equation}
respectively. 
For simplicity, we use the affine noise schedule of \cite{lipman2023flow}
\[
\sigma_t \coloneqq 1 - (1- \sigma_{\min})t , \qquad \sigma_{\min} \in (0,1).
\]
Thus, the path starts at a reference noise, $X_0=Z$, and ends at a noisy target sample, $X_1=Y+\sigma_{\min}Z$. The results can be easily generalised to other noise schedules and mean interpolation weights \citep{lipman2023flow}.

In this setting, we can define mappings that connect the quantities introduced in \eqref{eq:path} in a point wise framework: For a fixed target value $y \in \mathcal{H}$, define the conditional flow as $\psi_t^y(z)\coloneqq ty+\sigma_t z$ with its derivative $\partial_t \psi_t^y(z)=y+\sigma_t^{\prime} z$. Since $\sigma_t>0$, the map $z\mapsto\psi_t^y(z)$ is invertible, with
$(\psi_t^y)^{-1}(x)=(x-ty)/\sigma_t$. Substituting gives the corresponding conditional velocity field
\begin{align*}
    v_t^{y}(x)=y+\sigma^{\prime}_t (\psi_t^y)^{-1}(x) = y+\frac{\sigma_t'}{\sigma_t}(x-ty)=
\frac{\sigma_t'}{\sigma_t}x
+
\left(1-t\frac{\sigma_t'}{\sigma_t}\right)y,
\end{align*}
which is the solution of the ODE $\partial_t\psi_t^y(z)=v_t^y(\psi_t^y(z))$. Since $Z=\frac{X_t-tY}{\sigma_t}$, we have for the random variables in \eqref{eq:path}
\[
v_t^Y(X_t)
=
Y+\sigma_t'\frac{X_t-tY}{\sigma_t}
=
Y+\sigma_t'Z
=
V_t.
\]

Since a point wise formulation in $t$ is too strong for practical purposes, we introduce the time-distribution $T\sim\eta$, which is independent of $(Y,Z)$, where
$\eta(\dd t)=w(t)\dd t$ is a probability measure on $[0,1]$
with $w>0$ almost everywhere.

In this setting, a candidate velocity field $u:[0,1]\times\mathcal H\to\mathcal H$ is considered a good approximation of $V_t$ if \begin{equation}
\label{eq:regress}
    \E_{T,Y,Z}\|u_T(X_T)-V_T\|_{\mathcal{H}}^2
\end{equation}
is {\it small}. Here $u_t(x)=u(t,x)$. The following proposition shows that \eqref{eq:regress} is minimised for the conditional mean velocity given the time and the current state. This is the Hilbert-space analogue of the least-squares characterisation known from nonparametric regression; see
\citet[Section~1.1, Equation~(1.1)]{gyorfi2002distribution} for the scalar case.

\begin{proposition}\label{prop:fm}
There exists $v\in L^2(dt\dd \mu_t;\mathcal H)\cap L^2(d\eta(t)\dd \mu_t;\mathcal H)$ such that, for Lebesgue a.e.~$t\in [0,1]$,
\begin{equation*}\label{eq:condexp}
 v_t(X_t)=\E[V_t\mid X_t]\quad\text{$\mathbb{P}$-almost surely}.
\end{equation*}
For every  $u\in L^2(\dd\eta(t)\mu_t;\mathcal{H})$,
\begin{align}\label{eq:regression}
 \E_{T,Y,Z}\|u_T(X_T)-V_T\|_{\mathcal{H}}^2
 &=\int\|u_t(x)-v_t(x)\|_{\mathcal{H}}^2\dd \mu_t(x)\dd\eta(t)+\E_{T,Y,Z}\|V_T-v_T(X_T)\|_{\mathcal{H}}^2.
\end{align}
Consequently, $v$ is the unique minimiser of \eqref{eq:regression} in $L^2(\dd\eta(t)\mu_t;\mathcal H)$.
Moreover, $t\mapsto\mu_t$ is continuous in $W_2$ and has endpoints $\mu_0$ and
\begin{equation*}\label{eq:endpoint}
\mu_1=\Law(Y+\sigma_{\min}Z),\qquad
 W_2^2(\mu_1,\nu)\leq\sigma_{\min}^2\Tr (\Sigma_\U).
\end{equation*}

\end{proposition}
\Cref{prop:fm} follows the lines of \cite{zhang2026flow} and recovers the classical properties of flow matching known from the Euclidean setting.

Evaluating the vector field in \Cref{prop:fm} point wise yields 
\begin{equation}\label{eq:v_t_with_expectation}
 v_t(x)=\frac{\sigma_t^{\prime}}{\sigma_t}x+
 \Big(1-t\frac{\sigma_t^{\prime}}{\sigma_t}\Big)\E[Y\mid X_t=x]
 \quad\text{for }\dd t\,\mu_t\text{-almost every }(t,x).
\end{equation} 

The optimal field $v$ still depends on the unknown target distribution $\nu$, but takes only time and current state as inputs. Therefore, the 
the population objective
\begin{equation*}
  F(u) \coloneqq   \int\|u_t(x)-v_t(x)\|_{\mathcal{H}}^2 \dd\mu_t(x)\dd\eta(t)
\end{equation*}
is not tractable.

\paragraph{Unpenalised empirical flow matching.}
We have only given an i.i.d.\ sample $Y^{(n)}=(Y_1,\ldots,Y_n)$ with $Y_1,\ldots,Y_n\overset{\mathrm{iid}}{\sim}\nu$, which we use to estimate $v_t$ with a prescribed class $\mathcal{V}$. We replace $\nu$ by the empirical
measure $\nu_n\coloneqq n^{-1}\sum_{i=1}^n\delta_{Y_i}$.
Let $I$ be uniformly distributed on $\{1,\ldots,n\}$ and let
$T\sim\eta$, with $I$, $T$, and $Z$ independent of each other
and of $Y^{(n)}$. Conditional on the observations, $Y_I$ has
distribution $\nu_n$. The empirical counterparts of
\eqref{eq:path} are therefore
\begin{equation}\label{eq:emp_Xt_and_Vt}
 X_t^n=tY_I+\sigma_tZ,\qquad
 V_t^n=Y_I+\sigma_t'Z,\qquad
 \mu_t^n=\Law(X_t^n\mid Y^{(n)}).
\end{equation}
The empirical flow matching loss is the empirical counterpart of \eqref{eq:regression}, i.e., the squared prediction error $\|u_T(X_T^n)-V_T^n\|_{\mathcal H}^2$, averaged over $I$, $T$, and $Z$. Applying \Cref{prop:fm} conditionally on $Y^{(n)}$, with target measure $\nu_n$, shows that the optimal predictor is the conditional-mean field $v^n$, characterised by
\begin{equation}\label{eq:emp_condexp}
 v_t^n(X_t^n)=\E[V_t^n\mid X_t^n,Y^{(n)}]
 \qquad\mathbb P\text{-almost surely}
\end{equation}
for Lebesgue almost every $t\in[0,1]$. For a candidate $u\in\mathcal V$, we measure its discrepancy from $v^n$ by
\begin{equation}\label{eq:Fn}
 F_n(u)\coloneqq
 \int_0^1\int_{\mathcal H}
 \|u_t(x)-v_t^n(x)\|_{\mathcal H}^2
 \dd \mu_t^n(x)\dd \eta(t).
\end{equation}
The identity in \Cref{prop:fm} yields
\begin{equation*}
 F_n(u)=\E\left[
 \|u_T(X_T^n)-V_T^n\|_{\mathcal H}^2
 \,\middle|\,Y^{(n)}\right]-c_n,
\end{equation*}
where $c_n\coloneqq\E[\|V_T^n-v_T^n(X_T^n)\|_{\mathcal H}^2\mid Y^{(n)}]\geq0$ is independent of $u$. Since $I$ is uniform on $\{1,\ldots,n\}$, this becomes
\begin{equation}\label{eq:emp-reg}
 F_n(u)=\frac1n\sum_{i=1}^n\int_0^1
 \E_Z\left[
 \left\|u_t(tY_i+\sigma_tZ)-(Y_i+\sigma_t'Z)\right\|_{\mathcal H}^2
 \right]\dd \eta(t)-c_n.
\end{equation}

\paragraph{Penalised empirical flow matching.}
In the setting of \cite{baldan2026physics} and also in the penalised diffusion setting of \cite{bastek2025physics}, the penalty evaluates $R$ at the endpoint predicted by a candidate field $u$ from a training state $X_t^n$. To formalise this in our Hilbert space setting, we first need to ensure that the image space of the residual is rich enough.

\begin{assumption}\label{ass:basic}
Let $(\Yres,\langle\cdot,\cdot\rangle_{\Yres})$ be a separable Hilbert space and let $R:\mathcal H\to\Yres$ be continuous. Define $\M\coloneqq R^{-1}(\{0\})$ and assume that $\nu(\M)=1$.
\end{assumption}
To make the endpoint evaluation of $R$ well defined, we impose the following regularity condition on $\mathcal V$.
\begin{assumption}\label{ass:flows}
Each $u\in\mathcal V$ is continuous as a function of $(t,x)$ on an open neighbourhood of $[0,1]\times\mathcal H$ and satisfies
\[
 \|u_t(x)-u_t(x')\|_{\mathcal H}
 \leq L_u(t)\|x-x'\|_{\mathcal H},
 \qquad x,x'\in\mathcal H,\quad t\in(0,1),
\]
where $L_u\in C((0,1);[0,\infty))\cap L^1(0,1).$
\end{assumption}

Under \Cref{ass:flows}, $\|u_t(x)\|_{\mathcal H}\leq\|u_t(0)\|_{\mathcal H}+L_u(t)\|x\|_{\mathcal H}$, with $t\mapsto u_t(0)$ continuous on $[0,1]$ and $L_u\in L^1(0,1)$. Hence the associated ODE running from time $s$ to $r$, where $0\leq s\leq r\leq1$ is well posed; see \citet[Proposition~18]{zhang2026flow}. For $x\in\mathcal H$, let $\gamma^{u,s}$ be the unique solution of
\begin{equation}\label{eq:ode_flow}
 \frac{\partial}{\partial q}\gamma_q^{u,s}(x)=u_q(\gamma_q^{u,s}(x)),
 \qquad \gamma_s^{u,s}(x)=x,
\end{equation}
and write $\Phi_{s\to r}^u(x)\coloneqq\gamma_r^{u,s}(x)$ for the state reached at time $r$. The trajectories are $C^1$, and Gr\"onwall's inequality \citep[Section~1]{Walter1970} gives
\begin{equation*}\label{eq:flow_lipschitz}
 \|\Phi_{s\to r}^u(x)-\Phi_{s\to r}^u(x')\|_{\mathcal H}
 \leq \exp\left\{\int_s^r L_u(q)\dd q\right\}
 \|x-x'\|_{\mathcal H}.
\end{equation*}
In particular, each flow map is continuous. Starting from $X_t^n$, the predicted endpoint is $\Phi_{t\to1}^u(X_t^n)$. We therefore define
\begin{equation*}
 J_n(u)\coloneqq
 \int_0^1\int_{\mathcal H}
 \|R(\Phi_{t\to1}^u(x))\|_{\Yres}^2
 \dd \mu_t^n(x)\dd \eta(t).
\end{equation*}
Equivalently, this is the sample average
\begin{equation}\label{eq:Jn}
 J_n(u)=\frac1n\sum_{i=1}^n\int_0^1
 \E_Z\left[
 \left\|R\left(\Phi_{t\to1}^u(tY_i+\sigma_tZ)\right)\right\|_{\Yres}^2
 \right]\dd \eta(t),
\end{equation}
where $\E_Z$ denotes expectation with respect to reference noise only. The penalty is thus measured at the predicted endpoint, rather than at the noisy training state.

\paragraph{Training objective and generated distribution.}
Combining \eqref{eq:emp-reg} and \eqref{eq:Jn}, we recover the penalised empirical risk minimiser
\begin{equation}\label{eq:estimator}
 \widehat u_{n,\lambda}
 \in\argmin_{u\in\mathcal V}
 \bigl\{F_n(u)+\lambda J_n(u)\bigr\}.
\end{equation}
Once the field has been fitted, a new sample is generated by following its flow from the independent reference distribution $\mu_0$ from time $0$ to $1$. For any $u\in\mathcal V$, write
\begin{equation}\label{eq:Pu}
 P^u\coloneqq(\Phi_{0\to1}^u)_\#\U
 =\Law(\Phi_{0\to1}^u(Z)).
\end{equation}
For the fitted field, $P^{\widehat u_{n,\lambda}}$ denotes this distribution conditionally on the training sample. Generation therefore starts from reference noise, while the training penalty evaluates endpoints reached from empirical interpolation states.

\section{Geometric criteria for improved generation}\label{sec:geometry}
The constraint penalty $J_n$ encourages predicted endpoints to satisfy the known constraint. However, even if newly generated samples lie close to the constraint set, their distribution need not be closer to the target. A generator may, for example, produce admissible samples but concentrate them near a single point, while the target is spread across the constraint set. 

We first examine this question at the level of the generated distributions. Let $P\in\Ptwo(\mathcal H)$ denote the distribution generated by a fitted velocity field and let
$\nu\in\Ptwo(\mathcal H)$ be the unknown distribution of the observed data. We recall that $\mathcal{M} = \{x \in \mathcal{H}\colon R(x) = 0\}$ and assume $\nu(\mathcal M)=1$ for a nonempty closed constraint set $\mathcal M\subset\mathcal H$. We measure the generation accuracy through the squared $2$-Wasserstein distance $W_2^2(P,\nu)$, defined as
\begin{equation*}
    W_2^2(P,\nu)\coloneqq\inf _{\xi \in \Gamma(P, \nu)} \int_{\mathcal{H} \times \mathcal{H}}\|x-y\|^2 \dd \xi(x, y),
\end{equation*}
where $\Gamma(P, \nu)$ denotes the set of all couplings of $P$ and $\nu$, i.e., probability measures on $\mathcal{H} \times \mathcal{H}$ whose first and second marginals are $P$ and $\nu$, respectively.
This choice is natural as it reflects the geometry of the underlying space and remains informative even when the generated and target distributions are mutually singular. The analysis depends only on the distributions and therefore applies to flow matching as well as other generative models.

To understand how constraints affect this error, we distinguish the distance of the generated samples from $\mathcal{M}$ from the distributional discrepancy that remains after nearest-point projection onto $\mathcal{M}$. Under the projection assumptions stated below, we relate these two quantities to the total $2$-Wasserstein error. We first establish bounds for a fixed distribution $P$, then
use them to compare the expected errors of penalised and unpenalised
estimators.

\paragraph{Geometric error components.} 
We measure the discrepancy of the sample to the constraint set via the $1$-Lipschitz distance function
\[
d_{\mathcal M}(x)\coloneqq\inf_{y\in\mathcal M}\|x-y\|_{\mathcal H}.
\]
To define the projected distribution, we explicitly assume the existence of a measurable nearest-point selection.

\begin{assumption}
\label{ass:measurable_proj}
    Assume that there is a Borel set $S\subseteq\mathcal H$ with $P(S)=1$ and a Borel measurable map $\pi:S\to\mathcal M$ such that
$\|x-\pi(x)\|_{\mathcal H}=d_{\mathcal M}(x)$ for every $x\in S$.
\end{assumption}

\begin{definition}\label{def:errors}
 Under \Cref{ass:measurable_proj}, we define 
\begin{equation*}
 \mathcal{N}(P)\coloneqq \int d_\M(x)^2\dd P(x),\quad 
 \mathcal{T}(P, \nu)\coloneqq W_2^2(\pi_\#P,\nu).
\end{equation*}
We call $\mathcal{N}$ the normal and $\mathcal{T}$ the tangential error.
We define 
\begin{equation}\label{eq:Cdef}
     \mathcal{C}(P, \nu)\coloneqq W_2^2(P,\nu)-\mathcal{N}(P)-\mathcal{T}(P, \nu),
\end{equation}
which we refer to as the remainder.
\end{definition}
The remainder $\mathcal C$ accounts for the interaction between tangential and normal error. If the nearest-point selection is not unique, the quantities involving $\pi$ depend on that choice. To interpret these quantities, imagine moving each generated sample to its nearest point on the constraint set. The normal error measures the average squared displacement required to make the samples admissible. The tangential error measures the distributional discrepancy that remains after this correction. For example, if the target is supported on a horizontal line, projection removes the vertical displacement, but the projected samples may still have the wrong spread along the line. 

\paragraph{Affine constraint sets.}
The decomposition is simplest when $\mathcal M$ is affine. Then the projection is orthogonal and the remainder vanishes.
\begin{proposition}\label{cor:flat}
Let $P\in\Ptwo(\mathcal H)$ denote the generated distribution and $\nu\in\Ptwo(\mathcal H)$ be the target distribution, which both satisfy the assumptions in \Cref{def:errors}. If $\mathcal M$ is a affine subspace of $\mathcal H$, then
\begin{equation}\label{eq:decomp_geometric}
W_2^2(P,\nu)=\mathcal N(P)+\mathcal T(P,\nu).
\end{equation}
\end{proposition}
Thus reducing the distance to $\mathcal M$ improves distributional accuracy precisely when this gain outweighs any increase in the error of the projected distribution.

\paragraph{Curved constraint sets.}
For a general constraint set, the remainder \eqref{eq:Cdef} need not vanish. The direction from the projection to another admissible point may have a component along the projection displacement. The following assumption controls the departure from the affine case through the geometry of $\mathcal M$.
In what follows, $\mathbb{B}$ denotes the open unit ball and $x+\alpha \mathbb{B}$ the ball of radius $\alpha>0$ and centre $x$.  For a subset $S$ of the Hilbert space and $\alpha>0$ we set $S+\alpha \mathbb{B}=\bigcup_{s\in S} s+ \alpha \mathbb{B}$. 

\begin{assumption}\label{ass:tubular}
There exist $\tau>0$, $\chi\geq0$ and a Borel measurable map
$\pi:\mathcal M+\tau\mathbb B\to\mathcal M$ such that every
$x\in\mathcal M+\tau\mathbb B$ has a unique nearest point
$\pi(x)\in\mathcal M$. Moreover, for every
$x\in\mathcal M+\tau\mathbb B$, $y\in\mathcal M$ and Borel measurable map $\pi:S \to \mathcal{M}$, 
\begin{equation}\label{eq:reach_assumption}
\big|
\langle x-\pi(x),y-\pi(x)\rangle_{\mathcal H}
\big|
\leq
\frac{\chi}{2}
d_{\mathcal M}(x)
\|y-\pi(x)\|_{\mathcal H}^2.
\end{equation}
\end{assumption}
The first condition guarantees that the metric projection onto $\mathcal M$
is well defined in a neighbourhood of $\mathcal M$. The second controls the
deviation from the orthogonality relation associated with a flat set. Indeed,
if $\mathcal M$ is an affine subspace, then
$x-\pi(x)$ is orthogonal to $y-\pi(x)$ and the left-hand side of
\eqref{eq:reach_assumption} is zero. More generally,
\eqref{eq:reach_assumption} bounds this deviation quadratically in the
distance along $\mathcal M$, with $\chi$ controlling the local geometry of
the set.

The following theorem quantifies the interaction between the normal and tangential error under this assumption.

\begin{theorem}\label{thm:geometric}
Grant \Cref{ass:measurable_proj}. Then
\begin{equation}\label{eq:crude_geometry}
\big|\mathcal C(P,\nu)\big|
\leq
2\sqrt{\mathcal N(P)\mathcal T(P,\nu)}.
\end{equation}

If, in addition, \Cref{ass:tubular} holds and for some $r<\tau$ with $\chi r<1$,
$P(\{x\in\mathcal H:d_{\mathcal M}(x)\leq r\})=1$, then
\begin{equation}\label{eq:sharp_geometry}
\mathcal N(P)+(1-\chi r)\mathcal T(P,\nu)
\leq
W_2^2(P,\nu)
\leq
\mathcal N(P)+(1+\chi r)\mathcal T(P,\nu).
\end{equation}
Consequently,
$\big|\mathcal C(P,\nu)\big|
\leq \chi r\,\mathcal T(P,\nu)$.
\end{theorem}
The theorem shows that the product \(\chi r\) controls how accurately the sum of the normal and tangential errors describes the total distributional error. Here, \(\chi\) bounds the departure from affine geometry, while \(r\) bounds the distance of generated samples from the constraint set. When this product is small, the affine decomposition remains approximately valid.

\paragraph{Comparing learned distributions.}
We now apply the geometric bounds to distributions estimated from observations. Let $\widehat P_{n,\lambda}\in\Ptwo(\mathcal H)$
denote a random learned distribution based on $Y_1,\ldots,Y_n$,
where $\lambda=0$ denotes the unpenalised reference and $\lambda>0$
a penalised procedure. For the flow-matching estimator
\eqref{eq:estimator}, this is
$\widehat P_{n,\lambda}=P^{\widehat u_{n,\lambda}}.$ 
As the geometric bounds depend only on their output distribution, the result below applies equally to other distribution estimators satisfying the stated assumptions. 
Its expected squared $2$-Wasserstein errors is
\[
\Risk_n(\lambda)\coloneqq
\E\!\left[W_2^2\!\left(\widehat P_{n,\lambda},\nu\right)\right],
\]
where the expectation averages over the observations and any additional randomness in learning. Whenever the following quantities are integrable, we write
$\mathsf N_\lambda\coloneqq\E\mathcal N(\widehat P_{n,\lambda})$,
$\mathsf T_\lambda\coloneqq\E\mathcal T(\widehat P_{n,\lambda},\nu)$, and
$\mathsf C_\lambda\coloneqq\E\mathcal C(\widehat P_{n,\lambda},\nu)$.

\begin{corollary}\label{cor:geocomparison}
Assume that a common measurable nearest-point selection is defined almost
surely under $\widehat P_{n,0}$ and
$\widehat P_{n,\lambda}$. Then $\Risk_n(\lambda)<\Risk_n(0)$ if and only if $\mathsf N_0-\mathsf N_\lambda>\mathsf T_\lambda-\mathsf T_0+\mathsf C_\lambda-\mathsf C_0$.
If, in addition, \Cref{ass:tubular} holds and for some
$r<\tau$ with $\chi r<1$,
\[
\widehat P_{n,0}(\{d_{\mathcal M}\leq r\})
=
\widehat P_{n,\lambda}(\{d_{\mathcal M}\leq r\})
=
1
\]
almost surely, then a sufficient condition for
$\Risk_n(\lambda)<\Risk_n(0)$ is
\[
\mathsf N_0-\mathsf N_\lambda
>
\mathsf T_\lambda-\mathsf T_0
+
\chi r(\mathsf T_0+\mathsf T_\lambda).
\]
If $\mathcal M$ is a closed affine subspace, then
$\Risk_n(\lambda)<\Risk_n(0)$ if and only if
$\mathsf N_0-\mathsf N_\lambda>\mathsf T_\lambda-\mathsf T_0$.
\end{corollary}

These criteria identify the changes in the generated distribution that are sufficient for penalisation to improve accuracy. They do not yet establish whether minimising the penalised training objective produces those changes. In the next section, we address this question for an explicit velocity-field class and determine how the resulting improvement depends on sample size and penalty weight.

\section{Penalised flow matching under affine hypothesis classes}\label{sec:estimator_based_analysis}

We now investigate how the penalty in \eqref{eq:estimator} changes the fitted velocity field and the distribution it generates. We restrict on affine velocity field classes for which both the training objective and the generated distribution can be computed explicitly. Within this class, the generated distributions share a fixed shape and differ only in their learned location. The comparison of their $2$-Wasserstein errors therefore reduces to comparing the accuracy of the estimated location, leading to a mean-estimation problem.

For affine penalties, we obtain a closed-form penalised estimator and an exact comparison of its expected $2$-Wasserstein error with its unpenalised counterpart.

\subsection{Affine residuals}\label{sec:linear}
Since the population vector field is given by \eqref{eq:v_t_with_expectation}, a simple class $\mathcal V$ can be obtained by retaining its known dependence on $(t,x)$ and replacing the conditional expectation by a simple constant map. 
\begin{assumption}\label{ass:vector_field_affine} Let $\mathcal E$ be a Hilbert space  and let $W:\mathcal E\to\mathcal H$ be a bounded linear operator  such that $W^*W\ge c_W^2I_{\mathcal{E}}$ for some  $c_W> 0$. We choose
    \begin{equation*}
\mathcal{V} \coloneqq \Big\{u^{\theta}\colon [0,1] \times \mathcal{H}\rightarrow \mathcal{H} \mid u_t^\theta(x)
 =\frac{\sigma_t'}{\sigma_t}x
  +\Big(1-t\frac{\sigma_t'}{\sigma_t}\Big)W\theta,
 \, \theta\in \mathcal{E},\, 0\le t\le 1\Big\}.
\end{equation*}
\end{assumption}

We further we impose the following assumption on the residual.

\begin{assumption}\label{ass:residual}
    The residual $R$ is given by
$  R(x) = Ax -f,$ 
    where $A\colon \mathcal{H} \rightarrow \mathcal{Y}$ is a bounded linear operator and $f \in \mathcal{Y}$. Furthermore, for $Y \sim \nu$, we have $AY = f$ almost surely.  
\end{assumption}

We keep the training-time distribution $\eta$ arbitrary and collect its contribution in the quantities
\begin{equation}\label{eq:kappa}
\kappa_\sigma
\coloneqq
\int_0^1
\Big(1-t\frac{\sigma_t'}{\sigma_t}\Big)^2\dd \eta(t)
\leq
\frac{1}{\sigma_{\min}^2}
<\infty, \quad \text{and} \quad \beta_\sigma
\coloneqq
\int_0^1
\left(\frac{1-t}{\sigma_t}\right)^2
\dd \eta(t),
\end{equation}
where we used $1-t\sigma_t'/\sigma_t=1/\sigma_t$. For the linear noise schedule in \Cref{ass:basic}, $\sigma_t\geq 1-t$, and hence $\beta_\sigma\leq1$. For other schedules, the corresponding bound may depend on $\sigma_{\min}$. Our analysis will reveal that the quantities $\kappa_\sigma$ and $\beta_\sigma$ determine the intrinsic relative scaling of the flow matching and penalty terms.

Under \Cref{ass:residual}, the minimiser \eqref{eq:estimator} can be computed explicitly for every $\lambda\geq0$. The following theorem states its form and shows how the balance between squared
bias and estimation variance determines whether penalisation
improves generation accuracy.

\begin{theorem}\label{thm:linear_linear_exact}
Under Assumptions \ref{ass:basic}, \ref{ass:vector_field_affine} and \ref{ass:residual}, set
\begin{equation}\label{eq:def_G_lambda}
G_\lambda
\coloneqq
\kappa_\sigma W^*W
+
\lambda\beta_\sigma W^*A^*AW.
\end{equation}
Then \eqref{eq:estimator} has the unique minimiser
\[
\widehat\theta_{n,\lambda}
=
G_\lambda^{-1}
\left(
\kappa_\sigma W^*\bar Y_n
+
\lambda\beta_\sigma W^*A^*f
\right).
\]
Fix
$B_\lambda^2\coloneqq
\|W\E[\widehat\theta_{n,\lambda}]-m_\nu\|_{\mathcal H}^2$
and
$V_\lambda\coloneqq
\kappa_\sigma^2
\|WG_\lambda^{-1}W^*\Sigma_\nu^{1/2}\|_{\operatorname{HS}}^2$.
Then
\begin{align}
\label{eq: equivalence}
    \Risk_n(\lambda)<\Risk_n(0)
\quad\text{if and only if}\quad
B_\lambda^2-B_0^2
<
\frac{V_0-V_\lambda}{n}.
\end{align}
\end{theorem}
Since $W\theta_0$ is the orthogonal projection of $m_\nu$ onto
$\overline{\operatorname{ran}(W)}$, where
$\theta_\lambda\coloneqq\E[\widehat\theta_{n,\lambda}]$, we have
\[
B_\lambda^2-B_0^2
=
\|W(\theta_\lambda-\theta_0)\|_{\mathcal H}^2
\geq0.
\]
Thus penalisation cannot reduce the population bias relative to the
unpenalised estimator. \Cref{thm:linear_linear_exact} shows that an
improvement in risk is possible only if this increase in squared bias is
compensated by a sufficiently large reduction in variance.

For fixed $\sigma_{\min}$ and $\lambda>0$, the variance contribution is of
order $n^{-1}$. Hence, whenever $B_\lambda^2>B_0^2$, the penalised estimator
is eventually worse as $n\to\infty$. This illustrates why explicit penalties can be particularly useful
when observations are scarce. Then the reduction in estimation variance
can outweigh the bias introduced by penalisation. With more
observations, the unpenalised estimator estimates the location more accurately, and a fixed increase in bias eventually dominates. If $\sigma_{\min}$ also depends on
$n$, the effective penalty weight
$\lambda\beta_\sigma/\kappa_\sigma$ also changes. The increase in
squared bias may then decrease with $n$, so its magnitude must be compared with the variance reduction for each sample size.

We finally compare this bias--variance decomposition with the geometric
decomposition of \Cref{sec:geometry}. The two decompositions are different,
but in the affine setting their terms admit a direct comparison.
Let $\Pi\coloneqq P_{\ker(A)}$ be the orthogonal projection onto
$\ker(A)$ and set $N\coloneqq I_{\mathcal H}-\Pi$.

\begin{proposition}\label{prop:connection}
Under the assumptions of \Cref{thm:linear_linear_exact}, there exist
constants $K_N,K_T\geq0$, independent of $\lambda$, such that
\begin{align*}
\mathsf N_\lambda
&=
K_N
+
\big\|N\big(W\E[\widehat\theta_{n,\lambda}]-m_\nu\big)\big\|_{\mathcal H}^2
+
\frac{\kappa_\sigma^2}{n}
\big\|NWG_\lambda^{-1}W^*\Sigma_\nu^{1/2}\big\|_{\operatorname{HS}}^2,\\
\mathsf T_\lambda
&=
K_T
+
\big\|\Pi\big(W\E[\widehat\theta_{n,\lambda}]-m_\nu\big)\big\|_{\mathcal H}^2
+
\frac{\kappa_\sigma^2}{n}
\big\|\Pi WG_\lambda^{-1}W^*\Sigma_\nu^{1/2}\big\|_{\operatorname{HS}}^2.
\end{align*}
Consequently,
\[
\mathsf N_0-\mathsf N_\lambda
-
(\mathsf T_\lambda-\mathsf T_0)
=
\frac{V_0-V_\lambda}{n}
-
(B_\lambda^2-B_0^2).
\]
\end{proposition}

Proposition~\ref{prop:connection} shows that both the normal and
tangential errors contain contributions from squared bias and
estimation variance. The constants $K_N$ and $K_T$ reflect the
prescribed distributional shape and cancel when comparing penalty
weights. The geometric criterion of \Cref{sec:geometry} therefore
expresses the same improvement criterion as the bias--variance comparison in \Cref{thm:linear_linear_exact}.

For a finite-dimensional parameter space $\dim(\mathcal E)=p<\infty$, the comparison admits a coordinatewise
representation. Set $U\coloneqq W^*W$,
$V\coloneqq U^{-1/2}W^*A^*AWU^{-1/2}$, and
$\rho_\lambda\coloneqq\lambda\beta_\sigma/\kappa_\sigma$. Here the operator $U$ records how parameters changes translate into changes of the location $W\theta$. After normalising these directions, $V$ measures their effect on the constraint residual. The effective penalty weight $\rho_{\lambda}$ describes the strength of the penalty relative to the flow matching loss. Let $(e_j)_{j=1}^p$ be an orthonormal eigenbasis of $V$, with
$Ve_j=d_je_j$. Define
$g_j\coloneqq
\left\langle
U^{-1/2}W^*A^*(f-AWU^{-1}W^*m_\nu),
e_j
\right\rangle_{\mathcal E}$
and
$\gamma_j\coloneqq
\left\langle
U^{-1/2}W^*\Sigma_\nu WU^{-1/2}e_j,
e_j
\right\rangle_{\mathcal E}$. The eigenvalue $d_j$ measures how strongly the constraint residual
responds to a displacement in the corresponding normalised
direction. The quantity $\gamma_j/n$ is the variance of the
unpenalised location estimate in that direction.
Finally, $g_j$ describes the extent to which the penalty pushes
the unpenalised population fit in that direction. Although the target satisfies the constraint, its projection onto
the chosen location class need not do so, so $g_j$ can be nonzero. The following proposition
determines when reducing sampling variability outweighs the
additional bias caused by this displacement.

\begin{proposition}\label{prop:eigenbasis_euclidean_space}
Under the notation above, for every $\lambda>0$,
$\Risk_n(\lambda)<\Risk_n(0)$ 
if and only if
\[
\sum_{j=1}^p
\frac{\rho_\lambda g_j^2}
     {(1+\rho_\lambda d_j)^2}
<
\frac1n
\sum_{j=1}^p
\frac{(2d_j+\rho_\lambda d_j^2)\gamma_j}
     {(1+\rho_\lambda d_j)^2}.
\]
\end{proposition}

In the normalised coordinates $U^{1/2}\theta$, penalisation
multiplies sampling fluctuations in direction $e_j$ by
$(1+\rho_\lambda d_j)^{-1}$, while shifting the population fit
by $\rho_\lambda g_j/(1+\rho_\lambda d_j)$.
The criterion compares the resulting variance reduction with
the additional squared bias, summed over all directions.
Directions with $d_j=0$ remain unchanged. Penalisation is therefore
beneficial when it suppresses substantial sampling variability
without displacing the population fit too far.

For any fixed sample size and noise schedule, a sufficiently small
positive penalty improves accuracy whenever
$\sum_{j=1}^p d_j\gamma_j>0$. In this case, the variance reduction
is first order in $\rho_\lambda$, whereas the increase in squared
bias is second order. 

\subsection{Arbitrary residuals}
\label{sec:affine}
  For $\lambda = 0$, \Cref{thm:linear_linear_exact} reveals that the flow matching estimator recovers the linear regression estimator $\widehat{\theta}_{n,0}=(W^* W)^{-1} W^* \bar{Y}_n$ on a vector field level. Moreover, the generated distributions have the same shape and differ only through the location parameter $W\theta$. This yields an explicit comparison of their $2$-Wasserstein errors.

\begin{proposition}\label{thm:vt_affine}
Under \Cref{ass:basic} and \Cref{ass:vector_field_affine}, let
$\bar Y_n\coloneqq n^{-1}\sum_{i=1}^nY_i$. For each $\theta\in\mathcal E$, let $u^\theta$ denote the affine
velocity field defined in \Cref{ass:vector_field_affine}. Then
\begin{equation}\label{eq:Fn-affine}
F_n(u^\theta)
=
\kappa_\sigma\|W\theta-\bar Y_n\|_{\mathcal H}^2+C_n,
\end{equation}
where $C_n$ does not depend on $\theta$. 
Finally, the distribution in \eqref{eq:Pu} is
$P^{u^\theta}=\Law(W\theta+\sigma_1Z)$, and
\begin{equation}\label{eq:W2-affine}
W_2^2(P^{u^\theta},\nu)
=
\|W\theta-\mathbb{E}[Y]\|_{\mathcal H}^2+K_\sigma,
\end{equation}
where $K_\sigma\coloneqq
W_2^2(\Law(\sigma_1Z),(x\mapsto x-\mathbb{E}[Y])_\#\nu)$ is independent of
$\theta$. 
\end{proposition}
\Cref{thm:vt_affine} reduces the empirical flow matching objective to a
linear least-squares problem in $\theta$. We next compare the penalised and unpenalised solutions within the class
$\mathcal V$, in the same spirit as the geometric comparison of
\Cref{sec:geometry}. Compared to \Cref{thm:linear_linear_exact}, we do not impose any additional assumption on the residual.

\begin{corollary}\label{cor:linear_comparison}
Let $\widehat\theta_{n,\lambda}$ be any minimiser of \eqref{eq:estimator} in the setting of \Cref{ass:vector_field_affine} and define
$
 \delta_{n,\lambda}=\widehat\theta_{n,\lambda}-\widehat\theta_{n,0}.
$
Then
\begin{equation*}
\Risk_n(\lambda)<\Risk_n(0) \quad\text{if and only if}\quad  -2\E\ip{W\widehat{\theta}_{0, \lambda} -  \mathbb{E}[Y],W\delta_{n,\lambda}}>
 \E\norm{W\delta_{n,\lambda}}^2.
\end{equation*}
\end{corollary}

\Cref{cor:linear_comparison} provides a geometric characterisation when penalisation improves generation accuracy. 

The component of the displacement induced by the penalty in the
direction of the population mean must dominate the quadratic cost of that
displacement. In particular, components orthogonal to
$\mathbb{E}[Y]-W\widehat\theta_{n,0}$ can only increase the risk.

\section{Penalised flow matching in nonlinear hypothesis classes with fixed features}\label{sec:last_layer}
In this subsection, we allow the velocity field to depend
nonlinearly on the spatial variable. We represent this dependence
through fixed nonlinear features and estimate their output
coefficients. This includes shallow neural networks whose hidden weights and biases are fixed before training, while only the outer weights are learned. Choosing the hidden parameters randomly according to a prescribed distribution gives a random-feature model \citep{rahimi2008weighted}.

We use this
setting as a first nonlinear extension of \Cref{ass:vector_field_affine}
for approximating the conditional expectation in
\eqref{eq:v_t_with_expectation} while keeping the optimisation explicit. For ease of notation, we impose a linear assumption on the residual.
\begin{assumption}\label{ass:residual_linear}
    The residual is given by $R(x)=Ax$, where $A:\mathbb R^d\to\Yres$ is linear. Further it holds that $AY=0$ almost surely for $Y\sim\nu$.
\end{assumption}
Under \Cref{ass:residual_linear}, $\mathcal M=\ker(A)$. Let
$\Pi$ be the orthogonal projection onto $\ker(A)$ and set
$N\coloneqq I_{\mathbb R^d}-\Pi$. Decomposing the input with these projections enables the following choice of the set $\mathcal{V}$.

\begin{assumption}\label{ass:vector_field_features}
Let $\phi:\ker(A)\to\mathbb R^p$ and
$\psi:\ker(A)^\perp\to\mathbb R^q$ be fixed Borel feature maps with finite Lipschitz constant, with
$\psi(ar)=a\psi(r)$ for every $a\geq0$. We choose the velocity field class
\[
\mathcal V
\coloneqq
\left\{
u^{B,D}:
u_t^{B,D}(x)
=
\frac{\sigma_t'}{\sigma_t}x
+
\left(1-t\frac{\sigma_t'}{\sigma_t}\right)
\bigl(B\phi(\Pi x)+D\psi(Nx)\bigr)
\right\},
\]
where $B:\mathbb R^p\to\ker(A)$ and
$D:\mathbb R^q\to\ker(A)^\perp$ are linear.     
\end{assumption}
The decomposition through $\Pi$ and $N$ separates the component in
$\ker(A)$ from its orthogonal complement. The matrices $B$ and $D$ act on
these two components separately. Compared to
\Cref{ass:vector_field_affine}, the dependence on the state is nonlinear
through $\phi$ and $\psi$, while the optimisation remains linear in the
trainable coefficients. In particular, if $\phi$ and $\psi$ apply the ReLU function $\sigma(x)=\max\{x,0\}$ componentwise, then 
\[
B\phi(\Pi x)+D\psi(Nx)
=B\sigma(\Pi x)+D\sigma(Nx)
\]
is a vector-valued one-hidden-layer ReLU network with fixed hidden weight matrix $\Pi$ and $N$ and trainable output coefficients $B$ and $D$.

For the residual term, we approximate the flow from time $t$ to $1$ by one
explicit Euler step,
\begin{equation}\label{eq:euler_step}
\widetilde\Phi_{t\to1}^{B,D}(x)
\coloneqq
x+(1-t)u_t^{B,D}(x).
\end{equation}
Thus, throughout this section we replace \eqref{eq:Jn} with
\begin{equation}\label{eq:Jn_euler}
    \tilde{J}_n(u)\coloneqq
 \int_0^1\int_{\mathcal H}
 \|R(\widetilde\Phi_{t\to1}^{B,D}(x))\|_{\Yres}^2
 \dd \mu_t^n(x)\dd \eta(t).
\end{equation}
 The generated distribution
used in the risk is still defined through the solution of the corresponding
ODE. For $\phi_{t,i}\coloneqq\phi(tY_i+\sigma_t\Pi Z)$, we impose the following
conditions.

\begin{assumption}\label{ass:last_layer_optimization}
The expectations below are finite and
\[
\frac1n\sum_{i=1}^n
\int_0^1
\frac{w(t)}{\sigma_t^2}
\E_Z[\phi_{t,i}\phi_{t,i}^*]\dd t
\succ0,
\qquad
\E_Z[\psi(NZ)\psi(NZ)^*]\succ0.
\]
\end{assumption}

The next result shows that a sufficiently small penalty can strictly improve
the $2$-Wasserstein risk.

\begin{theorem}\label{thm:last_layer_improvement} Grant \Cref{ass:residual_linear}, \ref{ass:vector_field_features} and \ref{ass:last_layer_optimization} and assume that
$\E_Z[NZ\psi(NZ)^*]\neq0$. Then there exists $c>0$ such that, as
$\lambda\downarrow0$,
\[
\Risk_n(\lambda)
=
\Risk_n(0)
-
c\,\lambda
+
\mathcal O(\lambda^2).
\]
Consequently, for every $n$ there exists $\lambda_n^*>0$ such that
$\Risk_n(\lambda)<\Risk_n(0)$ for every $0<\lambda<\lambda_n^*$.
\end{theorem}

The improvement for small $\lambda$ comes from the penalty acting on the dynamics perpendicular
to the constraint space. Without penalisation, the fitted normal
coefficient is $D=0$, so the normal component of a generated sample
is $\sigma_{\min}NZ$. Although the observations satisfy the
constraint exactly, generation therefore retains noise outside
the constraint space.

The penalty introduces a correction through $D$, while leaving
the fitted tangential coefficient $B$ unchanged. The observations $Y^{(n)}$ on the other hand only act on $B$. The only components influencing both, $B$ and $D$ are the latent distribution and the noise schedule $\sigma_t$. This interaction is enough such that for larger $\lambda$, the penalty is not necessarily beneficial, but might impact the distribution estimation negatively. We note that in this setting, the precise effect can be calculated explicitly for a given choice of $\mu_0$, $\sigma_{\min}$, $A$ and $\psi$. \Cref{thm:last_layer_improvement} is stated for fixed $\sigma_{\min}$, but the proof can also be used to align $\lambda$ and $\sigma_{\min}$.

\section{Discussion and limitations}\label{sec:discussion}
Our work provides a first theoretical underpinning of penalised flow matching and identifies conditions under which explicit penalties can lead to improved generation performance. \Cref{sec:geometry} provides geometric conditions that decompose the overall $2$-Wasserstein performance measure in a normal error, a tangential error and a remainder as given in \Cref{def:errors}. For affine constraint sets, generation improves precisely when the reduction in the normal outweighs the increase in the tangential error. For curved sets, an additional remainder accounts on their interaction, and bounds on this remainder yield sufficient conditions for improvement.

 In the mean-estimation setting of \Cref{sec:estimator_based_analysis}, improved generation occurs when the reduction in estimation variance exceeds the increase in squared bias. For a fixed noise schedule, a fixed penalty that introduces a strictly positive bias eventually becomes harmful as the sample size increases. This highlights the potential benefit of explicit penalties in low-sample regimes, particularly relevant for scientific applications in which training data require costly numerical simulations \citep{peherstorfer2018survey}.

 \Cref{sec:last_layer} establishes a different improvement mechanism
for a nonlinear model with fixed features. Within the chosen architecture, a sufficiently small penalty reduces the remaining noise perpendicular to the constraint space while leaving the normal components separate and thus unchanged. This improvement holds for
every fixed sample size under the stated assumptions.

Together, these results show that the effect of a constraint penalty
depends on its interaction with the velocity-field class, the noise schedule and the available observations. In particular, the smoothing
induced by flow matching can itself leave samples outside the constraint set, such that an explicit penalty can lead to improved generation. On the other hand, the effect of the penalty are highly involved. Considered use requires a carefully assessed cost-benefit analysis.

This work is only the first step in a profound understanding of explicit penalties within generative models and provides several interesting directions for further research: 
\begin{itemize}
\item A very natural extension is to analyse settings in which information is imperfect, e.g., observations are contaminated and thus fulfil the constraint only approximately or the residual map itself underlies some uncertainty. In these regimes penalties might either compensate for additional noise or worsen the generation performance by steering the data towards a wrong constraint set. Another very interesting case are penalties that the newly generated samples should fulfil, but the available training data do not. This could be particularly interesting for, e.g., enforcing moral or legal restrictions on outputs. In view of the dangers of uncontrolled AI output, this question will become increasingly relevant. Here, the major challenge is how to formalise such a constraint.  

    \item Neural networks are known to learn data structures implicitly such that explicit penalties appear less beneficial in general. However, this only holds when the network architecture is carefully chosen and adapts to the data structure on which the networks are trained. Reversing the perspective implies that penalties might compensate for badly chosen architectures that are unable to capture the underlying data structure. Embedding structural information in the penalty might positively influences feature learning.    

    \item Constraint penalties may also improve statistical accuracy by reducing overfitting. Such benefits were already observed in the context of physics-informed neural networks (PINNs) \citep{doumeche2024kernel} and have been empirically proven for physics-informed diffusion models \citep{bastek2025physics}. However, explicit penalties such as physical laws alone need not prevent overfitting. The analysis of PINNs by \citet{doumeche2025convergence} shows that additional regularisation, independent of the specific constraint, can be necessary. In generative modelling, a model may likewise reproduce admissible training samples without accurately recovering the underlying distribution.

    \item While our results do not provide any statistical rates yet, the bias-variance decomposition in \Cref{thm:linear_linear_exact} provides the starting point. Following the results known for vanilla diffusion \citep{oko2023diffusion} similar statistical rates could be derived under structural assumptions such as factorisable densities \citep{kwon2026factorizable} or manifold support \citep{azangulov2024}. For penalised generative models we would expect faster convergence rates for small-sample regimes.
    
    \item  While we assume a general time distribution $\eta$, studying optimal choices of this distribution is particularly relevant for penalised settings, since the precision of the penalty evaluation during training heavily depends on the starting point for a fixed Euler scheme. Even in the unpenalised flow matching setting, it has been observed that the uniform distribution is not always the best choice \citep{esser2024scaling}. Thus, a rigorous analysis could lead to a significantly improved performance. 

    \item  While our geometric comparison of \Cref{sec:geometry} can assess outputs from any generative method (given the constraint assumption is fulfilled), \Cref{sec:estimator_based_analysis} focuses on penalties applied to the training objective, inspired by physics-informed diffusion models or flow matching  \citep{baldan2026physics, bastek2025physics}. Alternative approaches suggest to perform generation in transformed coordinates and map the generated samples back through a function whose outputs necessarily satisfy the constraint \citep{NEURIPS2023_85f5c737}. Constrained generative dynamics define the process directly on a constraint set or introduce reflections at the boundaries of an admissible domain \citep{ICLR2024_d1f9936d, pmlr-v202-lou23a}, while generator guidance steers sampling updates using gradients of the constraint during training \citep{Bansal_2023_CVPR}. To the best of our knowledge, theoretical underpinning of any of these methods is still lacking, limiting a direct comparison of which method to choose.

\end{itemize}

\paragraph{Acknowledgement.} The work of L. Kunkel and S. Langer was partially supported by the Deutsche Forschungsgemeinschaft (DFG) under project number
460867398 (DFG {\it Research unit 5381 Mathematical Statistics in the Information Age}).

\appendix

\section{Proofs}
\paragraph{Notation.} Random variables are defined on a common probability space $(\boldsymbol{\Omega},\mathcal A,\mathbb P)$. For a map $f:D\to\mathcal H$, where $D\subseteq\mathcal H$ and
$\mathcal H$ is a Hilbert space, we denote its Lipschitz constant by
\[
\operatorname{Lip}(f)
\coloneqq
\sup_{\substack{x,y\in D\\x\neq y}}
\frac{\|f(x)-f(y)\|_{\mathcal H}}{\|x-y\|_{\mathcal H}}.
\]
The map $f$ is Lipschitz continuous if
$\operatorname{Lip}(f)<\infty$.

\subsection{Proofs of \Cref{sec:fm}}\label{sec:proofs_fm}
\begin{proof}[Proof of \Cref{prop:fm}]
This first claim corresponds to \citet[Appendix A.1]{zhang2026flow}, which in turn relies on \citet[Proposition 1.10]{Da_Prato_Zabczyk_2014} and \citet[Lemma 1.14]{Kallenberg2021-jb}. We find that
\begin{equation}\label{eq:V_t_finite}
    \mathbb{E}[\|V_t\|_{\mathcal{H}}^2] =    \mathbb{E}[\|Y + \sigma_t^{\prime Z}\|_{\mathcal{H}}^2] = \mathbb{E}[\|Y\|_{\mathcal{H}}^2] + (1- \sigma_{\min})^2 \Tr(C) < \infty
\end{equation}
and similarly
\begin{equation*}
    \mathbb{E}\|X_t\|_{\mathcal{H}}^2=t^2 \mathbb{E}\|Y\|_{\mathcal{H}}^2+\sigma_t^2 \Tr(C)<\infty.
\end{equation*}
Therefore $\mathbb{E}[V_t|X_t]$ exists.
For a fixed $t$, we obtain by \citet[Lemma 1.14]{Kallenberg2021-jb} that there exists a Borel function $v_t \colon \mathcal{H} \rightarrow \mathcal{H}$ such that
\begin{equation}\label{eq:kallenberg_existence}
    v_t(X_t) = \mathbb{E}[V_t|X_t].
\end{equation}
Next, we introduce an additional uniformly distributed random variable $T_0 \sim \operatorname{Unif}([0,1])$, independent of $Y$ and $Z$. To do so, we extend the underlying probability space $(\Omega, \mathcal{A}, \mathbb{P})$ to the product space
$
(\boldsymbol{\Omega} \times[0,1], \mathcal{A} \otimes \mathcal{B}([0,1]), \mathbb{P} \otimes \lambda),
$
where $\lambda$ denotes Lebesgue measure on $[0,1]$ and $\mathcal{B}([0,1])$ is the Borel sigma-algebra on $[0,1]$. On this enlarged space, we define
$
T_0(\omega, s)=s
$
and extend $Y$ and $Z$ by setting $Y(\omega, s)=Y(\omega)$ and $Z(\omega, s)=Z(\omega)$. This preserves the joint law of $(Y, Z)$ and makes $T_0$ independent of $(Y, Z)$. Then $(T_0, X_{T_0})$ and $V_{T_0}$ are $\mathcal{A} \otimes \mathcal{B}([0,1])$-measurable random variables taking values in $[0,1] \times \mathcal{H}$ and $\mathcal{H}$, respectively, equipped with their Borel sigma-algebras. By the same arguments leading to the existence of the conditional expectation and the Borel function in \eqref{eq:kallenberg_existence}, we conclude that there is a map $v \colon [0,1] \times \mathcal{H}  \rightarrow \mathcal{H}$ such that
\begin{equation}\label{eq:v_t_unif}
    v(T_0, X_{T_0}) = \mathbb{E}[V_{T_0}|T_0, X_{T_0}], \qquad \text{ almost surely}.
\end{equation}
Then we obtain 
\begin{equation*}
    \int_0^1 \int_{\mathcal{H}}\|v(t, x)\|_{\mathcal{H}}^2 \dd \mu_t( x) \dd t < \infty.
\end{equation*}
Next we want to verify that $v_t(X_t) = \mathbb{E}[V_t|X_t]$ for almost every $t$. Let $B \in \mathcal{B}([0,1])$, $  A \in \mathcal{B}(\mathcal{H})$ and $  h \in \mathcal{H} .$ Then we obtain by \eqref{eq:v_t_unif} and Fubini's theorem
\begin{equation*}
    \int_B \mathbb{E}[\mathbf{1}_A\left(X_t\right)\langle v(t, X_t)-V_t, h\rangle] \dd t=0.
\end{equation*}
Hence the integral with respect to $t$ over every set $B \in \mathcal{B}([0,1])$ of $\mathbb{E}[\mathbf{1}_A\left(X_t\right)\langle v(t, X_t)-V_t, h\rangle]$ is zero for all $t$ up to a Lebesgue null set $N_{A, h}$. Note that the set $N_{A, h}$ can depend on $A$ and $h$. Choosing a countable algebra $\mathcal{A}$ generating $\mathcal{B}(\mathcal{H})$ and a countable (or finite) orthonormal basis $(e_j)$ of $\mathcal{H}$ and collecting all negligible  sets $N_{A, e_j}$ in the set $N_{0}$, we conclude that for every $t \notin N_{0}$, every $A \in \mathcal{A}$ and $j$,
\begin{equation*}
    \mathbb{E}[\mathbf{1}_A(X_t)\langle v(t, X_t)-V_t, e_j\rangle]=0.
\end{equation*}
Since $\mathcal{A}$ is a generating algebra, we conclude using \citet[Lemma 1.18]{Kallenberg2021-jb}  that 
\begin{equation*}
    v_t(X_t)=\mathbb{E}[V_t \mid X_t] \quad \text { almost surely, for every } t \notin N_0 .
\end{equation*}
Using \eqref{eq:V_t_finite} and the conditional version of Jensens inequality results in
\begin{equation*}
    \int_{\mathcal{H}}\left\|v_t(x)\right\|_{\mathcal{H}}^2 \mu_t(\mathrm{~d} x)=\mathbb{E}\left\|\mathbb{E}\left[V_t \mid X_t\right]\right\|_{\mathcal{H}}^2 \leq \mathbb{E}\|Y\|_{\mathcal{H}}^2+\left(1-\sigma_{\min }\right)^2 \operatorname{Tr}(C)<\infty
\end{equation*}
Thus we obtain
\begin{equation*}
    \int_0^1 \int_{\mathcal{H}}\|v_t(x)\|_{\mathcal{H}}^2 \dd\mu_t( x) \dd \eta(t) = \int_0^1 w(t)\Big[\int_{\mathcal{H}}\|v_t(x)\|_{\mathcal{H}}^2 \dd \mu_t( x)\Big] \dd t < \infty
\end{equation*}
since $w$ is the Radon-Nikodym derivative of the probability measure $\eta$ with respect to the Lebesgue measure. This yields $v \in L^2(\dd \eta(t)\mu_t;\mathcal{H})$.

Since 
\begin{equation*}
    \mathbb{E}[V_t-v_t(X_t) \mid X_t ] =  \mathbb{E}[V_t \mid X_t ] - v_t(X_t) = 0,
\end{equation*}
we have 
\begin{align*}
    \E_{T\sim\eta,Y,Z}&\norm{u_T(X_T)- v_T(X_T) + v_T(X_T)- V_T}^2 \\&=\int\norm{u_t(x)-v_t(x)}^2\mu_t(dx)\dd\eta(t) + \E_{T\sim\eta,Y,Z}\norm{V_T-v_T(X_T)}^2,
\end{align*}
where integrability is assured by assumption for the first integral and follows from $\mathbb{E}[\|V_t\|_{\mathcal{H}}]< \infty$. $\mu_1=\Law(Y+\sigma_{\min}Z)$ follows by construction. Since $\mathbb{E}[\|Y\|_{\mathcal{H}}]< \infty$ and $\mathbb{E}[\|Z\|_{\mathcal{H}}]< \infty$, we derive
\begin{equation*}
    W_2^2(\mu_1, \nu)\leq \mathbb{E}[\|Y + \sigma_{\min}Z - Y\|^2_{\mathcal{H}}] = \sigma_{\min}^2 \Tr(C).\qedhere
\end{equation*}
\end{proof}

\subsection{Proofs of \Cref{sec:geometry}}

\begin{proof}[Proof of \Cref{cor:flat}]
  If $\mathcal M$ is a closed affine subspace, then $\pi$
is the orthogonal projection onto $\mathcal M$.
In particular, for every $x\in\mathcal H$ and
$y\in\mathcal M$, the Pythagorean identity yields
\[
\|x-y\|_{\mathcal H}^2
=
\|x-\pi(x)\|_{\mathcal H}^2
+
\|\pi(x)-y\|_{\mathcal H}^2.
\]
Consequently,  the proof of Theorem~\ref{thm:geometric} below
applies {\it mutatis mutandis}, with $\chi=0$
and without any restriction on the support of $P$.

\end{proof}

\begin{proof}[Proof of \Cref{thm:geometric}] In this proof, we omit the index $\mathcal{H}$ in norms and inner products.
First, note that
\begin{equation*}
    W_2^2(P, \pi_{\#}P) = \inf_{\gamma \in \Gamma(P, \pi_{\#}P )}\int \|x-y\|^2 \dd \gamma (x, y) \geq \int \inf_{z \in \mathcal{M}} \|x-z\|^2 \dd P(x) = \mathcal{N}(P) 
\end{equation*}
and 
\begin{equation*}
    \inf_{\gamma \in \Gamma(P, \pi_{\#}P )}\int \|x-y\|^2 \dd \gamma (x, y) \leq \mathbb{E}[\|X- \pi(X)\|^2]= \mathcal{N}(P).
\end{equation*}
Thus $ W_2^2(P, \pi_{\#}P) = \mathcal{N}(P)$.
By triangle inequality
\begin{equation*}
    W_2(P, \nu) \leq  W_2(P, \pi_{\#}P) +  W_2(\pi_{\#}P, \nu) = \sqrt{\mathcal{N}(P)} + \sqrt{\mathcal{T}(P, \nu)}.
\end{equation*}
Reverse triangle inequality further yields
\begin{equation*}
     |W_2(P, \pi_{\#}P) -  W_2(\pi_{\#}P, \nu)| \leq W_2(P, \nu).
\end{equation*}
Squaring both sides yields the lower bound 
\begin{equation*}
    W^2_2(P, \pi_{\#}P) - 2 W_2(P, \pi_{\#}P)W_2(\pi_{\#}P, \nu) + W^2_2(\pi_{\#}P, \nu) \leq W^2_2(P, \nu)
\end{equation*}
and the upper bound
\begin{equation*}
    W^2_2(P, \nu) \leq W^2_2(P, \pi_{\#}P) + 2 W_2(P, \pi_{\#}P)W_2(\pi_{\#}P, \nu) + W^2_2(\pi_{\#}P, \nu).
\end{equation*}
Combining both bounds leads to
\begin{equation*}
  - 2 W_2(P, \pi_{\#}P) W_2(\pi_{\#}P, \nu) \leq \mathcal{C}(P, \nu )\leq  2 W_2(P, \pi_{\#}P) W_2(\pi_{\#}P, \nu),
\end{equation*}
which proves \eqref{eq:crude_geometry}.

For the second part of the claim, we start by calculating the Lipschitz constant of $\pi$ on $\mathcal{M}+r\mathbb{B}$. Let $x, x^{\prime} \in \mathcal{M}+r\mathbb{B}$. First, Cauchy-Schwarz inequality implies
\begin{equation*}
    \langle x-x^{\prime}, \pi(x) - \pi(x^{\prime}) \rangle \leq \|x-x^{\prime}\|\| \pi(x) - \pi(x^{\prime})\|.
  \end{equation*} Then, we add and subtract $\pi(x) - \pi(x^{\prime})$ and use \eqref{eq:reach_assumption} to obtain
\begin{align*}
     \langle x-x^{\prime}, \pi(x) - \pi(x^{\prime}) \rangle &= \| \pi(x) - \pi(x^{\prime}) \|^2 -    \langle x-\pi(x),   \pi(x^{\prime} - \pi(x)) \rangle -    \langle x^{\prime} - \pi(x^{\prime}), \pi(x) - \pi(x^{\prime}) \rangle\\
     & \geq \Big(1 - \frac{\chi}{2}(d_{\mathcal{M}}(x) + d_{\mathcal{M}}(x^{\prime})) \Big) \| \pi(x) - \pi(x^{\prime}) \|^2\\
     & \geq \Big(1 -\chi r  \Big)\| \pi(x) - \pi(x^{\prime}) \|^2,
\end{align*}
where the last inequality follows as $P(\{x\in\mathcal H:d_{\mathcal M}(x)\leq r\})=1$ by assumption.
Combining the upper and the lower bound, we obtain
\begin{equation*}
       (1 -\chi r ) \| \pi(x) - \pi(x^{\prime}) \|^2\leq     \|x-x^{\prime}\|\| \pi(x) - \pi(x^{\prime})\|.
\end{equation*}
Assume $\pi(x) \neq \pi(x^{\prime})$, then
\begin{equation*}
    \| \pi(x) - \pi(x^{\prime}) \|\leq \frac{1}{1-\chi r} \|x-x^{\prime}\|.
\end{equation*}
If $\pi(x) = \pi(x^{\prime})$ this holds trivially. Thus, $\pi \in \operatorname{Lip}\big(\frac{1}{1-\chi r} \big)$ on $U_r(\mathcal{M})$. This implies that $\pi$ is measurable.

For the lower bound, let $\gamma $ be the coupling that attains the lower bound in $W_2^2(P, \nu)$ and define $\bar{\gamma} = ( \pi, \operatorname{Id})_{\#}\gamma$. Note that existence follows from Lemma~8.1(a) in \cite{Bickel.Freedman.1981.AoS}. 
Using \eqref{eq:reach_assumption} leads to
\begin{align*}
    W_2^2(P, \nu) &= \int \|x- \pi(x) + \pi(x)-y\|^2\dd \gamma(x,y) \\&= \int \|x- \pi(x) \|^2\dd P(x ) + 2 \int \langle x- \pi(x) , \pi(x)-y \rangle + \|\pi(x)-y\|^2\dd \gamma(x,y)\\
    & \geq \mathcal{N}(P) + (1 - \chi r) \int\|z-y\|^2\dd \bar{\gamma}(z,y)\\
    & \geq \mathcal{N}(P) + (1 - \chi r) \mathcal{T}(P, \nu).
\end{align*}
For the upper bound, let $\gamma\in\Gamma(\pi_\#P,\nu)$
be an optimal coupling. By disintegration, there exists
$\gamma^*\in\Gamma(P,\nu)$ such that
$(\pi,\operatorname{Id})_\#\gamma^*=\gamma$. Then
\begin{align*}
    W_2^2(P, \nu) &\leq  \int \|x- \pi(x) + \pi(x)-y\|^2\dd \gamma^*(x,y) \\&= \int \|x- \pi(x) \|^2\dd P(x ) + 2 \int \langle x- \pi(x) , \pi(x)-y \rangle + \|\pi(x)-y\|^2\dd \gamma^*(x,y)\\
    & \leq  \mathcal{N}(P) + (1 + \chi r) \int\|z-y\|^2\dd \gamma(z,y)\\
    & \leq \mathcal{N}(P) + (1 + \chi r) \mathcal{T}(P, \nu).
\end{align*}
This leads to the desired bound in \eqref{eq:sharp_geometry}. 
\end{proof}

\begin{proof}[Proof of \Cref{cor:geocomparison}]
The first claim follows from the definition of the remainder \eqref{eq:Cdef}:
\begin{equation*}
    \mathbb{E}_{\mathcal{D}_n}[W_2^2(P^{\widehat{u}_{\lambda, n}}, \nu)] - \mathbb{E}_{\mathcal{D}_n}[W_2^2(P^{\widehat{u}_{0, n}}, \nu)] = (\mathsf{N}_\lambda-\mathsf{N}_0)+(\mathsf{T}_\lambda-\mathsf{T}_0)+(\mathsf{C}_\lambda-\mathsf{C}_0)
\end{equation*}
Hence, $\mathbb{E}_{\mathcal{D}_n}[W_2^2(P^{\widehat{u}_{\lambda, n}}, \nu)] - \mathbb{E}_{\mathcal{D}_n}[W_2^2(P^{\widehat{u}_{0, n}}, \nu)]<0$, if and only if
    \begin{equation*}
         \mathsf T_\lambda-\mathsf T_0+\mathsf C_\lambda-\mathsf C_0 < \mathsf N_0-\mathsf N_\lambda.
    \end{equation*}
For the second bound, we use that \eqref{eq:sharp_geometry} yields $|\mathsf C_\lambda | \leq \chi r \mathsf{T}_{\lambda}$,  for every $\lambda$, from where we deduce the second claim. 
\end{proof}

\subsection{Proofs of \Cref{sec:estimator_based_analysis}}
\begin{proof}[Proof of \Cref{thm:linear_linear_exact}]

First, we derive an explicit representation of the residual $R$ evaluated at $\Phi_{0 \rightarrow 1}^{\theta}(z)$. To this end, we start by showing that the solution of the ODE with $v^{\theta}$ starting from $s \in [0,1]$ and running to $r \in [0,1]$ with $s < r$,  
\begin{equation*}
    \frac{\partial}{\partial q}\gamma_q^{\theta,s}(x)=u_q(\gamma_q^{\theta,s}(x)),
 \qquad \gamma_s^{\theta,s}(x)=x
\end{equation*}
is given by
\begin{equation}\label{eq:solution_affine_field}
    \gamma_r^{\theta,s}(x)=r W \theta+\frac{\sigma_r}{\sigma_s}(x-s W \theta).
\end{equation}
The initial condition is fulfilled. Furthermore, taking the derivative yields
\begin{equation*}
     \frac{\partial}{\partial r} \gamma_r^{\theta,s}(x) = W\theta + \frac{\sigma_r^{\prime}}{\sigma_s} (x- s W \theta).
\end{equation*}
Inserting the proposed solution \eqref{eq:solution_affine_field} into the vector field structure leads to
\begin{align*}
    v_s^{\theta}(\gamma_r^{\theta,s}(x)) &= \frac{\sigma_r^{\prime}}{\sigma_r}\Bigl(r W\theta + \frac{\sigma_r}{\sigma_s}(x-sW\theta) \Bigr) + \Bigl( 1- r \frac{\sigma_r^{\prime}}{\sigma_r} \Bigr)W \theta\\
    & = \Big(1-   \frac{\sigma_r^{\prime}}{\sigma_s} s\Big)W\theta + \frac{\sigma_r^{\prime}}{\sigma_s}x,
 \end{align*}
 which verifies \eqref{eq:solution_affine_field}. Setting the starting time $t=0$ leads to
\begin{equation}\label{eq:solution_affine_field_0_1}
     \Phi_{0 \rightarrow 1}^{\theta}(z) = W\theta + \sigma_{\min}z.
 \end{equation}
    Under \Cref{ass:residual}, we can rewrite the residual using that $AY_i = f$ almost surely as
    \begin{align*}
        J_{n}(\theta) &= \frac{1}{n} \sum_{i = 1}^n \int_0^1 \mathbb{E}_Z \|\frac{1-t}{\sigma_t}(AW\theta -f) - \sigma_{\min}AZ\|_{\mathcal{Y}}^2\dd \eta(t)\\
        & = \beta_{\sigma} \|AW\theta -f\|_{\mathcal{Y}}^2 -    2    
        \frac{1}{n} \sum_{i = 1}^n \int_0^1 \mathbb{E}_Z \Big\langle \frac{1-t}{\sigma_t}(AW\theta -f),  \sigma_{\min}AZ\Big\rangle_{\mathcal{Y}}\dd \eta(t)\\
        & \quad + \sigma_{\min}^2 \mathbb{E}_{Z}[\|AZ\|^2].
    \end{align*}
    where the cross-term vanishes as $\mathbb{E}[Z] = 0$ . The last term does not depend on $\theta$ and is thus not relevant for the minimisation. 
    Denoting the sample mean by $\bar{Y}_n = \frac{1}{n} \sum_{i =1}^n Y_i$ and inserting the structure of the class $\mathcal{V}$ into the flow matching objective \eqref{eq:emp-reg} yields
\begin{align*}
   F_n(u^\theta)&=\frac{1}{n} \sum_{i = 1}^n \int_0^1 \mathbb{E}_Z[\|u_{t}^\theta(tY_i+\sigma_{t} Z)- (Y_i+\sigma_{t}^{\prime} Z)\|_{\mathcal{H}}^2] \dd \eta(t)\\
   & = \frac{1}{n} \sum_{i = 1}^n \int_0^1 \mathbb{E}_Z[\|(t \frac{\sigma_t^{\prime}}{\sigma_t} -1) (Y_i - W \theta)\|_{\mathcal{H}}^2] \dd \eta(t)\\
   & = \frac{\kappa_{\sigma}}{n} \sum_{i = 1}^n \|W \theta - Y_i\|_{\mathcal{H}}^2.
\end{align*}
Thus, minimising \eqref{eq:estimator} in the setting of this theorem is equivalent to minimising 
    \begin{equation}\label{eq:objective_linear_linear}
        Q_{n, \lambda}(\theta) \coloneqq  \frac{\kappa_{\sigma}}{n} \sum_{i = 1}^n \| W\theta - Y_i \|_{\mathcal{H}}^2 + \lambda \beta_{\sigma} \|AW\theta - f\|_{\mathcal{Y}}^2.
    \end{equation}
    Taking the directional derivative for $h \in \mathcal{E}$ leads to to
    \begin{align*}
        D_{\theta}Q_{n, \lambda}(\theta)[h] & = \lim_{\varepsilon \rightarrow 0} \frac{Q_{n, \lambda}(\theta + \varepsilon h) - Q_{n, \lambda}(\theta)}{\varepsilon}\\
        & = 2 \frac{\kappa_{\sigma}}{n} \sum_{i = 1}^n \langle W\theta - Y_i, Wh \rangle_{\mathcal{H}} + 2 \lambda \beta_{\sigma} \langle AW\theta - f, AW h\rangle_{\mathcal{Y}}\\
        & = 2 \kappa_{\sigma} \langle W\theta - \bar{Y}_n, Wh \rangle_{\mathcal{H}} + 2 \lambda \beta_{\sigma} \langle AW\theta - f, AW h\rangle_{\mathcal{Y}}.
    \end{align*} This expression is equal to $0$ for all $h \in \mathcal{E}$ precisely when
    \begin{equation}\label{eq:normal_equation_linear_linear}
        (\kappa_{\sigma} W^*W+ \lambda \beta_{\sigma} W^*A^*AW)\theta = \kappa_{\sigma}W^* \bar{Y}_n +\lambda \beta_{\sigma} W^*A^* f.
    \end{equation}
    We recall and define 
    \begin{equation*}
        G_{\lambda} \coloneqq  \kappa_{\sigma} W^*W+ \lambda \beta_{\sigma} W^*A^*AW, \quad H \coloneqq \kappa_{\sigma }
W^* \bar{Y}_n + \lambda \beta_{\sigma}W^*A^* f.
    \end{equation*}
    We note that $G_{\lambda} $ is self-adjoint and that for every $h \in \mathcal{E}$
    \begin{equation*}
        \langle G_{\lambda}h, h \rangle_{\mathcal{E}} = \kappa_{\sigma} \|Wh\|^2_{\mathcal{H}} + \lambda \beta_{\sigma} \|AWh\|^2_{\mathcal{Y}}\geq \kappa_{\sigma}c_W^2 \|h\|^2_{\mathcal{E}},
    \end{equation*}
    with $c_W$ from \Cref{ass:vector_field_affine}.
    Thus the inverse of $G_{\lambda}$ exists and is bounded. Hence, the the normal equation \eqref{eq:normal_equation_linear_linear} is fulfilled if and only if
    \begin{equation*}
      \theta = G_{\lambda}^{-1}H    =:  \widehat{\theta}_{n, \lambda}.
    \end{equation*}
    This is indeed the unique global minimiser of \eqref{eq:objective_linear_linear} since for every other $\theta \in \mathcal{E}$
    \begin{align*}
        Q_{n, \lambda}(\theta)  - Q_{n, \lambda}(\widehat{\theta}_{n, \lambda}) & = \langle G_{\lambda} \theta,\theta \rangle_{\mathcal{E}} - 2 \langle H,\theta \rangle_{\mathcal{E}} - \langle G_{\lambda} \widehat{\theta}_{n, \lambda},\widehat{\theta}_{n, \lambda} \rangle_{\mathcal{E}} + 2 \langle H , \widehat{\theta}_{n, \lambda}\rangle_{\mathcal{E}}\\
        & = \langle G_{\lambda} (\theta - \widehat{\theta}_{n, \lambda}), \theta - \widehat{\theta}_{n, \lambda}\rangle_{\mathcal{E}} \\
        & \geq \kappa_{\sigma}c_W^2 \|\theta - \widehat{\theta}_{n, \lambda}\|^2_{\mathcal{E}}, 
    \end{align*}
   and the  first claim follows. For the equivalence statement \eqref{eq: equivalence}, we perform a bias-variance decomposition of  the risk, i,e., 
    \begin{equation*}
        \mathcal{R}(\widehat{\theta}_{n,\lambda}) = \mathbb{E}[\|W \widehat{\theta}_{n,\lambda} - m_{\nu}\|_{\mathcal{H}}^2] = \mathbb{E}[\|W (\widehat{\theta}_{n,\lambda} - \mathbb{E}[\widehat{\theta}_{n,\lambda}])\|_{\mathcal{H}}^2]+ B_{\lambda}^2,  
    \end{equation*}
    where $B_\lambda$ is defined in the statement of the theorem. For the variance term, we use that
    \begin{equation*}
        \mathbb{E}[\widehat{\theta}_{n,\lambda}] = G_{\lambda}^{-1}(\kappa_{\sigma }W^*m_{\nu} +\lambda \beta_{\sigma} W^*A^* f)
    \end{equation*} and thus
    \begin{equation*}
        \mathbb{E}[\|W (\widehat{\theta}_{n,\lambda} - \mathbb{E}[\widehat{\theta}_{n,\lambda}])\|_{\mathcal{H}}^2] =\kappa_{\sigma}^2 \mathbb{E}[\|W (G_{\lambda}^{-1}W^*(\bar{Y}_n - m_{\nu}))\|_{\mathcal{H}}^2].
    \end{equation*}
    We have $\bar{Y}_n - m_{\nu} \in L^2$ and the covariance operator of $\bar{Y}_n$ is given by $\frac{1}{n}\Sigma_{\nu}$. {\it A fortiori}, there exist an orthonormal basis $(e_j)$ of $\mathcal{H}$ such that 
    \begin{align*}
         \mathbb{E}[\|W (G_{\lambda}^{-1}W^*(\bar{Y}_n - m_{\nu}))\|_{\mathcal{H}}^2] 
         &= \mathbb{E}\Big[ \sum_{j} \langle W (G_{\lambda}^{-1}W^*(\bar{Y}_n - m_{\nu})),e_j \rangle^2\Big]\\
         & =  \sum_{j}\mathbb{E}[ \langle \bar{Y}_n - m_{\nu}, W G_{\lambda}^{-1}W^* e_j \rangle^2]\\
        & =  \sum_{j} \Big\langle \frac{1}{n}\Sigma_{\nu} W G_{\lambda}^{-1}W^* e_j, W G_{\lambda}^{-1}W^* e_j \Big\rangle\\
         & =\Tr(WG_{\lambda}^{-1}W^*\frac{1}{n}\Sigma_{\nu}W G_{\lambda}^{-1}W^* )\\
         & = \frac{1}{n} \|W G_{\lambda}^{-1} W^* \Sigma_{\nu}^{1/2}\|_{\operatorname{HS}}^2.
    \end{align*}
    Here, we used that by Theorem 7.2.5 in \cite{Eubank2015-lg}, for $\rho \in \mathcal{P}_2(\mathcal{H})$ and $h,h^{\prime} \in \mathcal{H}$, $\langle\Sigma_\rho h, h^{\prime}\rangle=\mathbb{E}[\langle X-m_\rho, h\rangle\langle X-m_\rho, h^{\prime}\rangle]$. The same result also grants finiteness of the trace.
    Thus
    \begin{equation*}
        \mathcal{R}(\widehat{\theta}_{n,\lambda})  = \frac{1}{n} V_{\lambda} + B^2_{\lambda}.
    \end{equation*}
    Rearranging
    \begin{equation*}
        \mathcal{R}(\lambda) < \mathcal{R}(0) \quad \Longleftrightarrow \quad B^2_\lambda + \frac{V_{\lambda}}{n}<B^2_0 + \frac{V_{0}}{n}
    \end{equation*}
    yields the second claim. 
\end{proof}
\begin{proof}[Proof of \Cref{prop:connection}]
 For the first statement, we first note that
    \begin{equation*}
        Am_{\nu} = A \mathbb{E}[Y] = \mathbb{E}[AY] = f 
    \end{equation*}
    and thus
    \begin{equation*}
      0 =  Ax-f = A(x-m_{\nu}),
    \end{equation*}
    which implies that $\mathcal{M} = m_{\nu} + \ker(A)$. 
Since $\ker(A)$ is closed, we can decompose
    \begin{equation*}
        x- m_{\nu} = \Pi(x-m_{\nu}) + N(   x- m_{\nu}),
    \end{equation*}
    which implies that $\pi(x) = \Pi(x-m_{\nu}) +  m_{\nu}$ and
    $d_{\mathcal{M}}^2(x) = \|N(   x- m_{\nu})\|_{\mathcal{H}}^2$.  Using that the generated distribution is given by $P^{\theta} = \operatorname{Law}(W \theta + \sigma_{\min}Z)$, we get 
    \begin{equation*}
        \mathcal{N}(P^{\theta}) = \mathbb{E}_Z[\|N(W \theta + \sigma_{\min}Z) - Nm_{\nu}\|_{\mathcal{H}}^2] = \|N(W\theta - m_{\nu})\|_{\mathcal{H}}^2 + \sigma_{\min}^2\mathbb{E}_Z [\|NZ\|_{\mathcal{H}}^2],
    \end{equation*}
    where we use the notations introduced in \Cref{sec:geometry}.
    Furthermore, we recall from the proof of \Cref{thm:linear_linear_exact} that $P^{\theta} = \operatorname{Law}(W\theta + \sigma_{\min}Z)$. Combined with a shift by $m_{\nu}$, we obtain
    \begin{align*}
        \mathcal{T}(P^{\theta}, \nu) &= W_2^2(\pi_{\#}P^{\theta}, \nu) \\& =  W_2^2((\tau_{-m_{\nu}} \circ\pi)_{\#}P^{\theta}, (\tau_{-m_{\nu}})_{\#}\nu)\\&
        = \inf _{\gamma \in \Gamma\left(\operatorname{Law}\left(\sigma_{\min } \Pi Z\right),\left(\tau_{-m_\nu}\right)_{\#} \nu\right)} \int\left\|\Pi\left(W \theta-m_\nu\right)+z-u\right\|_{\mathcal{H}}^2 \dd \gamma(z, u)\\
        & = \|\Pi(W\theta - m_{\nu}) \|_{\mathcal{H}}^2 + \inf _{\gamma \in \Gamma\left(\operatorname{Law}\left(\sigma_{\min } \Pi Z\right),\left(\tau_{-m_\nu}\right)_{\#} \nu\right)} \int\|z-u\|^2 d \gamma(z, u),
    \end{align*}
    where $\tau_{-m_{\nu}}$ is the translation map defined via $\tau_{-m_{\nu}}(x)\coloneqq x-m_{\nu} $. In the last equality, we have used that both distributions are centred at $0$. Thus
    \begin{equation*}
        \mathsf{N}_0 - \mathsf{N}_{\lambda} = \mathbb{E} \|N(W\widehat{\theta}_{n, 0} - m_{\nu})\|_{\mathcal{H}}^2  - \mathbb{E} \|N(W\widehat{\theta}_{n, \lambda} - m_{\nu})\|_{\mathcal{H}}^2 
    \end{equation*}
    and \begin{equation*}
        \mathsf{T}_{\lambda} - \mathsf{T}_0 = \mathbb{E}\|\Pi(W\widehat{\theta}_{n, \lambda} - m_{\nu}) \|_{\mathcal{H}}^2 -  \mathbb{E}\|\Pi(W\widehat{\theta}_{n, 0} - m_{\nu}) \|_{\mathcal{H}}^2.
    \end{equation*}
The bias--variance decomposition used in the proof of
\Cref{thm:linear_linear_exact} applies, {\it mutatis mutandis},
to each $Q\in\{\Pi,N\}$. Retaining, {\it in toto}, the
notation introduced therein, we obtain
\begin{equation*}
\begin{aligned}
\mathbb{E}\|Q(W\widehat{\theta}_{n,\lambda}-m_\nu)\|_{\mathcal H}^2
&=
\|Q(W\mathbb{E}[\widehat{\theta}_{n,\lambda}]-m_\nu)\|_{\mathcal H}^2+
\kappa_\sigma^2
\mathbb{E}\|QWG_\lambda^{-1}W^*(\bar Y_n-m_\nu)\|_{\mathcal H}^2.
\end{aligned}
\end{equation*}
The variance term is evaluated by the same covariance identity
as in the proof of \Cref{thm:linear_linear_exact}, yielding
\begin{equation*}
\mathbb{E}\|QWG_\lambda^{-1}W^*(\bar Y_n-m_\nu)\|_{\mathcal H}^2
=
\frac{1}{n}
\|QWG_\lambda^{-1}W^*\Sigma_\nu^{1/2}\|_{\mathrm{HS}}^2.
\end{equation*}
We recall from the proof of \Cref{thm:linear_linear_exact} that
\begin{equation*}
     \mathcal{R}(\widehat{\theta}_{n,\lambda})  = \frac{1}{n} V_{\lambda} + B^2_{\lambda}
\end{equation*}
Combined with \eqref{eq:decomp_geometric} we obtain the second statement.
\end{proof}
\begin{proof}[Proof of \Cref{thm:vt_affine}]
We denote the sample mean by $\bar{Y}_n = \frac{1}{n} \sum_{i =1}^n Y_i$. In the proof of \Cref{thm:linear_linear_exact}, we already showed that
\begin{equation*}
   F_n(u^\theta)= \frac{\kappa_{\sigma}}{n} \sum_{i = 1}^n \|W \theta - Y_i\|_{\mathcal{H}}^2 = \kappa_{\sigma} \|W \theta -  \bar{Y}_n \|_{\mathcal{H}}^2 + \frac{\kappa_{\sigma}}{n} \sum_{i = 1}^n \|\bar{Y}_n  - Y_i\|_{\mathcal{H}}^2.
\end{equation*}
The second equality prepares a later step of the proof. Remark~2.19 in \cite{gabriel2019computational},
although stated for Euclidean spaces, extends, {\it mutatis mutandis,}
to our Hilbert-space setting. {\it Ipso facto,} there exists a constant
$K_{\rm shape}$, independent of $\theta$, such that the
$2$-Wasserstein error of the corresponding distribution satisfies
\begin{equation*}
 W_2^2(P^\theta,\nu)
 =
 K_{\rm shape}
 +
 \norm{W\theta-m_\nu}_{\mathcal{H}}^2.
\end{equation*}

\end{proof}
\begin{proof}[Proof of \Cref{prop:eigenbasis_euclidean_space}]
    We start by introducing $\tilde{W} = WU^{-1/2}$, which is an isometry, and $S_\lambda=(I_{\mathcal{E}}+\rho_\lambda V)^{-1}$. We are going to use that $S_\lambda=I_{\mathcal{E}}-\rho_\lambda S_\lambda V$ and that $S_{\lambda}$ is self-adjoint since $V$ is self adjoint. Inserting all quantities in \eqref{eq:def_G_lambda} leads to 
    \begin{equation}\label{eq:G_inverse}
        G_{\lambda} = \kappa_\sigma U^{1 / 2}\left(I_{\mathcal{E}}+\rho_\lambda V\right) U^{1 / 2}, \quad G_\lambda^{-1}=\frac{1}{\kappa_\sigma} U^{-1 / 2} S_\lambda U^{-1 / 2}.
    \end{equation}
    We recall from \Cref{thm:linear_linear_exact}
    \begin{equation*}
        \widehat\theta_{n,\lambda}
=
G_\lambda^{-1}
\left(
\kappa_\sigma W^*\bar Y_n
+
\lambda\beta_\sigma W^*A^*f
\right).
    \end{equation*}
    Then
    \begin{align*}
        W \mathbb{E}[\widehat{\theta}_{n, \lambda}]&=\tilde{W} S_\lambda(\tilde{W}^* m_\nu+\rho_\lambda \tilde{W}^* A^* f)\\
        & = \tilde{W}(I_{\mathcal{E}}-\rho_\lambda S_\lambda V) \tilde{W}^* m_\nu+\rho_\lambda \tilde{W} S_\lambda \tilde{W}^* A^* f\\
        & = \tilde{W} \tilde{W}^* m_\nu-\rho_\lambda \tilde{W} S_\lambda V \tilde{W}^* m_\nu+\rho_\lambda \tilde{W} S_\lambda \tilde{W}^* A^* f\\
        & = \tilde{W} \tilde{W}^* m_\nu + \rho_\lambda \tilde{W} S_\lambda (\tilde{W}^* A^* f - V \tilde{W}^* m_\nu).
    \end{align*}
    Inserting the definitions of $V$ and $\tilde{W}^*$ leads to
    \begin{align*}
        \tilde{W}^* A^* f-V \tilde{W}^* m_\nu & =   \tilde{W}^* A^* f - \tilde{W}^* A^* A \tilde{W} \tilde{W}^* m_\nu\\
        & = \tilde{W}^* A^*  ( f- A \tilde{W} \tilde{W}^* m_\nu)\\
        & = U^{-1 / 2} W^* A^*(f-A W U^{-1} W^* m_\nu).
    \end{align*}
    Thus
    \begin{equation*}
          W \mathbb{E}[\widehat{\theta}_{n, \lambda}] = W(W^* W)^{-1} W^* m_{\nu} + \rho_\lambda \tilde{W} S_\lambda  U^{-1 / 2} W^* A^*(f-A W U^{-1} W^* m_\nu).
    \end{equation*}
    Using the definition of $\tilde{W}$, we see that both components belong to $\operatorname{ran}(W)$ and
  $
         W(W^* W)^{-1} W^* m_\nu-m_\nu \in \operatorname{ran}(W)^{\perp}.
 $
    This leads to
    \begin{align*}
        B_\lambda^2&=\| W(W^* W)^{-1} W^* m_\nu-m_\nu+\rho_\lambda \tilde{W} S_\lambda g\|^2 \\&= \| W(W^* W)^{-1} W^*  m_\nu-m_\nu\|^2+\rho_\lambda^2\|\tilde{W} S_\lambda U^{-1 / 2} W^* A^*(f-A W U^{-1} W^* m_\nu)\|^2.
    \end{align*}
    Hence, since $\tilde{W}$ is an isometry
    \begin{align*}
        B_\lambda^2-B_0^2& =\rho_\lambda^2\|\tilde{W} S_\lambda U^{-1 / 2} W^* A^*(f-A W U^{-1} W^* m_\nu)\|_{\mathcal{H}}^2 \\ &= \rho_\lambda^2\| S_\lambda U^{-1 / 2} W^* A^*(f-A W U^{-1} W^* m_\nu)\|_{\mathcal{H}}^2. 
    \end{align*}
    Since $(e_j)_{j=1}^p$ is an orthonormal eigenbasis of $V$, we know that $(I_{\mathcal{E}}+\rho_\lambda V) e_j=(1+\rho_\lambda d_j) e_j$. Thus
    \begin{equation*}
    e_j =     S_\lambda(I_{\mathcal{E}}+\rho_\lambda V) e_j=(1+\rho_\lambda d_j) S_\lambda e,
    \end{equation*}
    which implies $S_\lambda e_j=\frac{1}{1+\rho_\lambda d_j} e_j$. This finally yields
    \begin{equation*}
        B_\lambda^2-B_0^2 = \rho_{\lambda}^2 = \Big\|\sum_{j=1}^p \frac{g_j}{1+\rho_\lambda d_j} e_j\Big\|_{\mathcal{E}}^2 = \sum_{j=1}^p \frac{g_j^2}{\left(1+\rho_\lambda d_j\right)^2}.
    \end{equation*}
    For the variance, substitute \eqref{eq:G_inverse} into $V_{\lambda}$ in \Cref{thm:linear_linear_exact}, which leads to
    \begin{align*}
        V_\lambda&=\kappa_\sigma^2\|W G_\lambda^{-1} W^* \Sigma_\nu^{1 / 2}\|_{\mathrm{HS}}^2 = \big\|W U^{-1 / 2} S_\lambda U^{-1 / 2} W^* \Sigma_\nu^{1 / 2}\big\|_{\mathrm{HS}}^2\\& = \| \tilde{W} S_\lambda \tilde{W}^* \Sigma_\nu^{1 / 2} \|_{\mathrm{HS}}^2 = \|  S_\lambda \tilde{W}^* \Sigma_\nu^{1 / 2} \|_{\mathrm{HS}}^2= \Tr(S_\lambda \tilde{W}{ }^* \Sigma_\nu \tilde{W} S_\lambda),
    \end{align*}
    where the penultimate equality follows from the fact that $\tilde{W}$ is an isometry. 
    Using the orthonormal eigenbasis $(e_j)_{j=1}^p$ and the fact that $S_{\lambda}$ is self adjoint, we obtain
    \begin{equation*}
        V_\lambda = \sum_{j=1}^p\big\langle S_\lambda \tilde{W}^* \Sigma_\nu \tilde{W} S_\lambda e_j, e_j\big\rangle =  \sum_{j=1}^p\big\langle \tilde{W}^* \Sigma_\nu \tilde{W} S_\lambda e_j, S_\lambda e_j\big\rangle. 
    \end{equation*}
    Using that $S_\lambda e_j=\frac{1}{1+\rho_\lambda d_j} e_j$ and inserting the definition of $\gamma_j$, we obtain
    \begin{equation*}
        V_\lambda = \sum_{j=1}^p \frac{1}{(1+\rho_\lambda d_j)^2}\big\langle\tilde{W}^* \Sigma_\nu \tilde{W} e_j, e_j\big\rangle_{\mathcal{E}} = \sum_{j=1}^p \frac{\gamma_j}{\left(1+\rho_\lambda d_j\right)^2}.
    \end{equation*}
    Thus
    \begin{equation*}
        V_0-V_\lambda = \rho_\lambda \sum_{j=1}^p \frac{(2 d_j+\rho_\lambda d_j^2) \gamma_j}{(1+\rho_\lambda d_j)^2}.
    \end{equation*}
    Inserting the expressions for $B_\lambda^2-B_0^2$ and $V_0-V_\lambda$ into the condition in \Cref{thm:linear_linear_exact} yields the claim. 
\end{proof}
\begin{proof}[Proof of \Cref{cor:linear_comparison}]
As shown in \Cref{thm:vt_affine}, 
\begin{equation*}
 W_2^2(P^\theta,\nu)
 =
 K_\sigma
 +
 \norm{W\theta-m_\nu}_{\mathcal{H}}^2.
\end{equation*}
Thus in our setting 
\begin{equation*}
    \mathcal{R}_n(\lambda) - \mathcal{R}_n(0) = \mathbb{E}[\|W\widehat{\theta}_{n, \lambda} - m_{\nu}\|_{\mathcal{H}}^2 - \|W\widehat{\theta}_{n,0} - m_{\nu}\|_{\mathcal{H}}^2].
\end{equation*}
Since
\begin{equation*}
    \|W(\widehat{\theta}_{n, \lambda}-\widehat{\theta}_{0, \lambda})  + W\widehat{\theta}_{0, \lambda} -  m_{\nu}\|_{\mathcal{H}}^2 = \| W \delta_{n, \lambda}\|_{\mathcal{H}}^2 + 2 \langle W \delta_{n, \lambda} ,  W\widehat{\theta}_{0, \lambda} -  m_{\nu}\rangle_{\mathcal{H}}  + \|W\widehat{\theta}_{0, \lambda} -  m_{\nu} \|_{\mathcal{H}}^2,
\end{equation*}
we derive
\begin{equation*}
 \Risk_n(\lambda)-\Risk_n(0)
 =\E\norm{W\delta_{n,\lambda}}^2
 +2\E\ip{W\widehat{\theta}_{0, \lambda} -  m_{\nu},W\delta_{n,\lambda}}.\qedhere
\end{equation*}
\end{proof}
\subsection{Proofs of \Cref{sec:last_layer}}
\begin{proof}[Proof of \Cref{thm:last_layer_improvement}]
First, we note that since $\sigma_t^{\prime} t - \sigma_t = -1$,
\begin{align*}
    u^{B, D}_t(X_t) - (Y_i + \sigma_t^{\prime}Z)  &= \frac{1}{\sigma_t}(B \phi(\Pi(tY_i + \sigma_tZ)) - Y_i + D\psi(N(tY_i + \sigma_tZ)) \\
    & = \frac{1}{\sigma_t} (\underbrace{B \phi(t Y_i + \sigma_t \Pi Z) - Y_i }_{\in \ker(A)}+ \underbrace{\sigma_t D\psi(NZ)}_{\in \ker(A)^{\perp}} ).
\end{align*}
Thus, we can decompose 
\begin{equation*}
    \|    u^{B, D}_t(X_t) - (Y_i + \sigma_t^{\prime}Z) \|^2 = \frac{\|B \phi(t Y_i + \sigma_t \Pi Z) -Y_i\|^2}{\sigma_t^2} + \|D\psi(NZ)\|^2.
\end{equation*}
Using the one-step Euler solution \eqref{eq:euler_step}, the residual in \eqref{eq:Jn_euler} is evaluated at
\begin{equation*}
    \tilde{\Phi}_{t \rightarrow 1}^{B, D}(x) = \frac{\sigma_{\min}}{\sigma_t}x + \frac{1-t}{\sigma_t}(B\phi(\Pi x) + D \psi(Nx)).
\end{equation*}
Since $B(\phi(\Pi x)  \in \ker(A)$, the residual function evaluated at $tY_i + \sigma_t Z$ is
\begin{equation*}
    A \tilde{\Phi}_{t \rightarrow 1}^{B, D}(t Y_i + \sigma_t Z) = A(\sigma_{\min}Z + (1-t) D \psi(NZ)).  
\end{equation*}
Thus, we can decompose the penalised flow matching objective into 
\begin{align*}
    F_n(u^{B, D}) + \lambda \tilde{J}_n(u^{B, D}) &= \frac{1}{n}\sum_{i = 1}^n \int_0^1 \frac{w(t)}{\sigma_t^2}\mathbb{E}_Z[\|B \phi_{t,i} - Y_i \|^2] \dd t+ \mathbb{E}_Z[\|D\psi(NZ)\|^2] \\ & \quad + \lambda \int_0^1 w(t)\mathbb{E}_Z[\|A(\sigma_{\min}Z + (1-t) D \psi(NZ))\|^2] \dd t,
\end{align*}
where $\phi_{t,i} \coloneqq \phi(t Y_i + \sigma_t \Pi Z) $. In this setting, $B$ and $D$ can be optimised separately and the optimisation of $D$ is not influenced by the set of observations $\mathcal{D}_n$. 
For the optimal choice of $B$, we define
\begin{equation*}
    Q_n(B)\coloneqq\frac{1}{n} \sum_{i=1}^n \int_0^1 \frac{w(t)}{\sigma_t^2} \mathbb{E}_Z\big[\|B \phi_{t, i}-Y_i\|_{\mathbb{R}^d}^2] \dd t.
\end{equation*}
For a linear map $\Delta \colon \mathbb{R}^p \rightarrow \ker(A)$, the directional derivative is given by
\begin{equation*}
    \begin{aligned}
D_B Q_n(B)[\Delta] & =\lim _{\varepsilon \rightarrow 0} \frac{Q_n(B+\varepsilon \Delta)-Q_n(B)}{\varepsilon} \\
& =\frac{2}{n} \sum_{i=1}^n \int_0^1 \frac{w(t)}{\sigma_t^2} \mathbb{E}_Z[\langle B \phi_{t, i}-Y_i, \Delta \phi_{t, i}\rangle_{\mathbb{R}^d}] \dd t\\
&=\frac{2}{n} \sum_{i=1}^n \int_0^1 \frac{w(t)}{\sigma_t^2} \mathbb{E}_Z[\langle(B \phi_{t, i}-Y_i) \phi_{t, i}^*, \Delta\rangle_{\mathrm{HS}}] \dd t\\
& = \Big\langle B\Big(\frac{1}{n} \sum_{i=1}^n \int_0^1 \frac{w(t)}{\sigma_t^2} \mathbb{E}_Z[\phi_{t, i} \phi_{t, i}^*] \dd t \Big)-\frac{1}{n} \sum_{i=1}^n \int_0^1 \frac{w(t)}{\sigma_t^2} \mathbb{E}_Z[Y_i \phi_{t, i}^*] \dd t, \Delta\Big\rangle_{\mathrm{HS}}
\end{aligned}
\end{equation*}
where for two linear operators $B_1, B_2$ we define $\langle B_1, B_2\rangle_{\mathrm{HS}}=\operatorname{Tr}(B_1^* B_2)$. Ensuring that $D_B Q_n(B)$ vanishes for every $\Delta$ is equivalent to 
\begin{equation*}
    B\frac{1}{n} \sum_{i=1}^n \int_0^1 \frac{w(t)}{\sigma_t^2} \mathbb{E}_Z\left[\phi_{t, i} \phi_{t, i}^*\right] \dd t = \frac{1}{n} \sum_{i=1}^n \int_0^1 \frac{w(t)}{\sigma_t^2} \mathbb{E}_Z\left[Y_i \phi_{t, i}^*\right] \dd t.
\end{equation*}
By assumption, the first matrix is invertible, thus we set
\begin{equation*}
    \widehat{B}_n \coloneqq \Big( \frac{1}{n} \sum_{i=1}^n \int_0^1 \frac{w(t)}{\sigma_t^2} \mathbb{E}_Z\left[Y_i \phi_{t, i}^*\right] \dd t \Big) \Big(\frac{1}{n} \sum_{i=1}^n \int_0^1 \frac{w(t)}{\sigma_t^2} \mathbb{E}_Z\left[\phi_{t, i} \phi_{t, i}^*\right] \dd t \Big)^{-1} .
\end{equation*}
Since $Q_n$ is a convex function in $B$, $\widehat{B}_n $ is the unique global minimiser.

For the optimisation of $D$, we take the same approach. Setting $a_1 \coloneqq \int_0^1 w(t)(1-t)\dd t $ and $a_2\coloneqq \int_0^1 w(t)(1-t)^2\dd t$, we define
\begin{equation*}
    G_{\lambda}(D) \coloneqq \mathbb{E}_Z\left[\|D \psi(N Z)\|^2\right]+\lambda \int_0^1 w(t) \mathbb{E}_Z\left[\left\|A\left(\sigma_{\text {min }} Z+(1-t) D \psi(N Z)\right)\right\|^2\right] \dd t.
\end{equation*}
With the same definition of $\langle \cdot, \cdot \rangle_{\mathrm{HS}},$ we obtain for every linear $\Delta \colon \mathbb{R}^q \rightarrow \ker(A)^{\perp} $
\begin{align*}
D_D G_\lambda(D)[\Delta]= & \lim _{\varepsilon \rightarrow 0} \frac{Q_\lambda(D+\varepsilon \Delta)-Q_\lambda(D)}{\varepsilon} \\
= & 2 \mathbb{E}_Z\left[\langle D \psi(N Z), \Delta \psi(N Z)\rangle_{\mathbb{R}^d}\right] \\
& +2 \lambda \int_0^1 w(t)(1-t) \mathbb{E}_Z\left[\left\langle A\left(\sigma_{\text {min }} Z+(1-t) D \psi(N Z)\right), A \Delta \psi(N Z)\right\rangle\right] \dd t \\
= & 2 \mathbb{E}_Z\left[\left\langle D \psi(N Z) \psi(N Z)^*, \Delta\right\rangle_{\mathrm{HS}}\right] \\
& +2 \lambda \int_0^1 w(t)(1-t) \mathbb{E}_Z\left[\left\langle A^* A\left(\sigma_{\text {min }} N Z+(1-t) D \psi(N Z)\right) \psi(N Z)^*, \Delta\right\rangle_{\mathrm{HS}}\right] \dd t \\
= & 2\left\langle D H+\lambda a_1 \sigma_{\text {min }} A^* A \mathbb{E}_Z[NZ\psi(NZ)]+\lambda a_2 A^* A D H, \Delta\right\rangle_{\mathrm{HS}},
\end{align*}
where $H \coloneqq \mathbb{E}[\psi(NZ)\psi(NZ)^*]$ is positive definite by assumption. Then, we obtain the normal equation 
\begin{equation*}
    DH + \lambda a_1 \sigma_{\min} A^*A\mathbb{E}_Z[NZ\psi(NZ)] + a_2 \lambda A^* A DH = 0, 
\end{equation*}
which results in 
\begin{equation*}
    \widehat{D} = - M_{\lambda} \mathbb{E}_Z[NZ\psi(NZ)]H^{-1}, \quad \text{where} \quad M_{\lambda} \coloneqq \lambda a_1 \sigma_{\min} (I_d + a_2 \lambda A^*A)^{-1}A^*A.
\end{equation*}
This is indeed the global minimiser, since
\begin{align*}
G_\lambda(\widehat{D}+\Delta)-G_\lambda(\widehat{D})= & 2\left\langle\widehat{D} H+\lambda a_1 \sigma_{\min } A^* A \mathbb{E}_Z[NZ\psi(NZ)]+\lambda a_2 A^* A \widehat{D} H, \Delta\right\rangle_{\mathrm{HS}} \\
& +\mathbb{E}_Z\left[\|\Delta \psi(N Z)\|^2\right]+\lambda a_2 \mathbb{E}_Z\left[\|A \Delta \psi(N Z)\|^2\right] \\
= & \Big\|\Delta H^{1 / 2}\Big\|_{\mathrm{HS}}^2+\lambda a_2\Big\|A \Delta H^{1 / 2}\Big\|_{\mathrm{HS}}^2 \\
\geq & \Big\|\Delta H^{1 / 2}\Big\|_{\mathrm{HS}}^2 .
\end{align*}
The lower bound $\|\Delta H^{1 / 2}\|^2_{\mathrm{HS}}$ is non-negative and zero if and only if $\Delta = 0$. In case $\lambda = 0$, the optimal choice is the zero-matrix.
Due to the construction of the class $\mathcal{V}$, the ODE evolves within $\ker(A)$ and within the complement separately:
\begin{equation*}
    \frac{\partial \Phi_t(x)}{\partial t} = \frac{\sigma_t^{\prime}}{\sigma_t}\Pi \Phi_t(x) + \frac{1}{\sigma_t} B\phi(\Pi \Phi_t(x) ) +  \frac{\sigma_t^{\prime}}{\sigma_t}N \Phi_t(x) + \frac{1}{\sigma_t} D\psi(N \Phi_t(x) ). 
\end{equation*}
Setting
\begin{equation*}
    \Phi_t^{B, D}(x)=\Phi_t^{\|, B}(\Pi x)+\Phi_t^{\perp, D}(N x)
\end{equation*}
where $\Phi_t^{\|, B}$ and $\Phi_t^{\perp, D}$ are the solutions of the ODEs
\begin{align}
    \frac{\partial}{\partial t} \Phi_t^{\|, B}(y)&=\frac{\sigma_t^{\prime}}{\sigma_t} \Phi_t^{\|, B}(y)+\frac{1}{\sigma_t} B \phi(\Phi_t^{\|, B}(y)), & &\Phi_0^{\|, B}(y)=y, y \in \ker(A) \notag \\
    \frac{\partial}{\partial t} \Phi_t^{\perp, D}(r)&=\frac{\sigma_t^{\prime}}{\sigma_t} \Phi_t^{\perp, D}(r)+\frac{1}{\sigma_t} D \psi(\Phi_t^{\perp, D}(r)), & &\Phi_0^{\perp, D}(r)=r, r \in \ker(A)^{\perp}.\label{eq:ODE_normal}
\end{align}
Since for $\lambda = 0$, the optimal choice $\widehat{D}$ is the zero matrix, the solution to the corresponding ODE is $\Phi_t^{\perp, 0}(r) = \sigma_t r$. The solution $\Phi_t^{\|, B}$ does not depend on $\lambda$. Via \eqref{eq:decomp_geometric}, we now conclude
\begin{equation}\label{eq:risk_normal}
    \mathcal{R}_n(\lambda) - \mathcal{R}_n(0) = \mathbb{E}[\|\Phi_1^{\perp, D_\lambda}(N Z)\|^2]-\sigma_{\min }^2 \mathbb{E}\|N Z\|^2.
\end{equation} 
To control the first term on the right-hand side,
we introduce an \emph{ad hoc} rescaling of the ODE
that removes the time dependence arising from
$\sigma_t'/\sigma_t$. 
To this end, we set $g(r) \coloneqq \mathbb{E}_Z[NZ\psi(NZ)]H^{-1} \psi(r)$ and consider a rescaled ODE
\begin{align*}
   \frac{\partial}{\partial t} \Big( \frac{\Phi_t^{\perp, D}(r)}{\sigma_t}\Big)  &= - \frac{\sigma_t^{\prime}}{\sigma_t^2}   \Phi_t^{\perp, D}(r) + \frac{1}{\sigma_t} \frac{\partial}{\partial t}\Phi_t^{\perp, D}(r)\\
   & =  - \frac{1}{\sigma_t}M_{\lambda} g \Big( \frac{\Phi_t^{\perp, D}(r)}{\sigma_t}\Big),
\end{align*} where the second equality follows from inserting the ODE \eqref{eq:ODE_normal} and using that $\psi(ar)=a\psi(r)$ for every $a\geq0$. Then we introduce for $t \in [0,1]$ the time-shift $s(t) = \int_0^{t} \sigma^{- \tau}\dd \tau$ and denote its inverse with $t(s)$, which exists since $\sigma_{\min}^{-1}>0$, accepting an abuse of notation. Then we can set
$V_s =\frac{\Phi_{t(s)}^{\perp, D}(r)}{\sigma_{t(s)}}$ and obtain for $0 \leq s \leq s(1) = \log(\sigma_{\min}^{-1})/(1- \sigma_{\min})$ the ODE
\begin{equation*}
    \frac{\partial}{\partial s} V_s =  \frac{\partial}{\partial t} \Big( \frac{\Phi_t^{\perp, D}(r)}{\sigma_t}\Big)\Big|_{t = t(s)} \frac{\partial t(s)}{\partial s} = - M_{\lambda}g(V_s).
\end{equation*}
Then setting $J_{\sigma} \coloneqq \log(\sigma_{\min}^{-1})/(1- \sigma_{\min})$, yields
\begin{equation}\label{eq:V_phi}
    \Phi_{t(s)}^{\perp, D}(r) = \sigma_{t(s)} V_{s}(r), \quad \Phi_{1}^{\perp, D}(r) = \sigma_{\min} V_{J_{\sigma}}(r).
\end{equation}
Now we want to bound the remainder
\begin{align}
    E_s(r)& \coloneqq V_s(r)-r+\lambda \sigma_{\min} a_1 s A^*A g(r) \notag \\
    & = - \int_0^s M_{\lambda} (g(V_s(r)) - g(r)) \dd r - s (M_{\lambda} - \lambda \sigma_{\min}a_1 A^*A)g(r).\label{eq:remainder}
\end{align}
Using that $\|(I_d+\lambda a_2 A^*A)^{-1}\| \leq 1$, we obtain $\|M_{\lambda}\|\leq \lambda a_1 \sigma_{\min} \|A^*A\|$. We note that $g$ is globally Lipschitz with constant $\operatorname{Lip}(g)$. Then
\begin{equation*}
    \|g(V_s(r)) - g(r)\| \leq \operatorname{Lip}(g) \|V_s(r) -r \|
\end{equation*}
Furthermore, we have that
\begin{equation*}
    V_s(r)  = r - \int_0^s M_\lambda g(V_\tau(r)) \dd \tau
\end{equation*}
and thus
\begin{equation*}
    \|V_s(r)\| \leq \|r\| + \lambda a_1 \sigma_{\min} \|A^*A\| \operatorname{Lip}(g) \int_0^s \|g(V_\tau(r))\|\dd \tau.
\end{equation*}
Grönwalls inequality \cite[Section 1]{Walter1970} yields
\begin{equation*}
    \|V_s(r)\|  \leq \exp(\lambda \sigma_{\min} a_1 \|A^*A\| \operatorname{Lip}(g) s )\|r\|.
\end{equation*}
This implies starting from \eqref{eq:remainder} and setting $C_{\exp} \coloneqq \exp( \sigma_{\min} a_1 \|A^*A\| \operatorname{Lip}(g) s )$ for $\lambda \in [0,1]$ that
\begin{equation*}
    \|V_s(r) -r\| \leq \|M_{\lambda}\| \operatorname{Lip}(g) \int_0^s \|V_\tau(r)\| \dd \tau \leq \|M_{\lambda}\| \operatorname{Lip}(g) J_{\sigma} C_{\exp} \|r\|.
\end{equation*}
Hence
\begin{equation}\label{eq:remainder_first_term}
   \Big\| \int_0^s M_{\lambda} (g(V_\tau(r)) - g(r)) \dd \tau \Big\| \leq \|M_{\lambda}\|^2 \operatorname{Lip}(g)^2 J_{\sigma} C_{\exp} \|r\|.
\end{equation}
For the second term in \eqref{eq:remainder}, we use that
\begin{equation*}
    M_{\lambda} - \lambda \sigma_{\min}a_1 A^*A = - \lambda^2 \sigma_{\min} a_1 a_2 (I_d + \lambda a_2 A^*A)^{-1}A^*A A^*A 
\end{equation*}
which results in
\begin{equation*}
    \|     M_{\lambda} - \lambda \sigma_{\min}a_1 A^*A\| \leq \lambda^2 \sigma_{\min} a_1 a_2 \|A^*A\|^2.
\end{equation*}
Then 
\begin{equation}\label{eq:remainder_second_term}
    s \|(M_{\lambda} - \lambda \sigma_{\min}a_1 A^*A)g(r) \| \leq J_{\sigma} \lambda^2 \sigma_{\min} a_1 a_2 \|A^*A\|^2 \operatorname{Lip}(g) \|r\|.
\end{equation}
Combining \eqref{eq:remainder_first_term} and \eqref{eq:remainder_second_term}, we obtain for the norm of \eqref{eq:remainder}
\begin{equation*}
    \sup_{0 \leq s \leq J_{\sigma}}\|E_s(r)\|\leq \lambda^2 ( a_1^2 \sigma^2_{\min} \|A^*A\|^2 \operatorname{Lip}(g)^2 J_{\sigma} C_{\exp} + J_{\sigma}  \sigma_{\min} a_1 a_2 \|A^*A\|^2 \operatorname{Lip}(g) )\|r\|.
\end{equation*}
Thus, setting $C_{\mathrm{rem}} =  a_1^2 \sigma^2_{\min} \|A^*A\|^2 \operatorname{Lip}(g)^2 J_{\sigma} C_{\exp} + J_{\sigma}  \sigma_{\min} a_1 a_2 \|A^*A\|^2 \operatorname{Lip}(g) $, we have established that 
\begin{equation*}
    V_s(r) = r - \lambda \sigma_{\min} a_1 s A^*A g(r) + E_s(r), \quad \text{with} \quad   \sup_{0 \leq s \leq J_{\sigma}}\|E_s(r)\|\leq C_{\mathrm{rem}} \lambda^2.
\end{equation*}
Using \eqref{eq:V_phi}, we obtain
\begin{align*}
    \mathbb{E}[\|\Phi_1^{\perp, D_\lambda}(N Z)\|^2] & = \sigma_{\min}^2  \mathbb{E}[\|V_{J_0}(NZ)\|^2]\\
    & = \sigma_{\min}^2 \Big(\mathbb{E}[\|NZ\|^2]  - 2 \lambda \sigma_{\min} a_1 J_{\sigma}\mathbb{E}[\langle NZ ,  A^*A g(NZ)\rangle ] - 2 \lambda \mathbb{E}[\langle NZ ,  E_s(NZ)\rangle ]\\
    & \quad +\mathbb{E}[ \|\lambda \sigma_{\min} a_1 J_{\sigma} A^*A g(NZ)  + E_s(NZ)\|^2]\Big).
\end{align*}
Since
\begin{equation*}
    2 \lambda \mathbb{E}[\langle NZ ,  E_s(NZ)\rangle ] \leq 2 \lambda\mathbb{E}[\| NZ\|  \|E_s(NZ)\| ]\leq 2 \lambda^2 C_{\operatorname{rem}} \mathbb{E}[\| NZ\|]
\end{equation*}
and
\begin{align*}
    \mathbb{E}[ \|\lambda \sigma_{\min} a_1 J_{\sigma} A^*A g(NZ) + E_s(NZ)\|^2]&\leq 2\lambda^2  \mathbb{E}[ \| \sigma_{\min} a_1 J_{\sigma} A^*A g(NZ) \|^2] + 2  \mathbb{E}[ \| E_s(NZ)\|^2]\\&\leq 2\lambda^2  \mathbb{E}[ \| \sigma_{\min} a_1 J_{\sigma} A^*A g(NZ) \|^2] + 2 C_{\operatorname{rem}}^2 \lambda^4,
\end{align*}
we conclude that
\begin{equation*}
    \mathbb{E}[\|\Phi_1^{\perp, D_\lambda}(N Z)\|^2] =\sigma_{\min}^2\big( \mathbb{E}[\|NZ\|^2] - 2 \lambda \sigma_{\min} a_1 J_{\sigma } \mathbb{E}[\langle NZ, A^*A g(NZ)] + \mathcal{O}(\lambda^2)\big).
\end{equation*}
Consequently, for \eqref{eq:risk_normal}
\begin{equation*}
    \mathcal{R}_n(\lambda) - \mathcal{R}_n(0) = -2 \lambda \sigma_{\min}^3 a_1 J_{\sigma } \mathbb{E}[\langle NZ, A^*A g(NZ)] + \mathcal{O}(\lambda^2).
\end{equation*}
The constant $c = 2  \sigma_{\min}^3 a_1 J_{\sigma } \mathbb{E}[\langle NZ, A^*A g(NZ)]$ is positive, since
\begin{align*}
    \mathbb{E}[\langle NZ, A^*A g(NZ)]& = \mathbb{E}_Z[\langle A N Z, A \mathbb{E}_Z\left[N Z \psi(N Z)^*\right] H^{-1} \psi(N Z)\rangle]\\
    & = \operatorname{Tr}\big(A \mathbb{E}_Z[N Z \psi(N Z)^*] H^{-1} \mathbb{E}_Z\left[\psi(N Z)(N Z)^*\right] A^*\big)\\
    & = \big\|A \mathbb{E}_Z[N Z \psi(N Z)^*] H^{-1 / 2}\big\|_{\mathrm{HS}}^2.
\end{align*}
Since $\mathbb{E}_Z[N Z \psi(N Z)^*]$ maps into $\ker(A)^{\perp}$ and $A$ is injective on $\ker(A)^{\perp},$ \Cref{ass:last_layer_optimization} grants that $c >0$.

\end{proof}

\section*{Miscellaneous}
For those who did not pay attention in their Latin lesson or just never attended one, we include a translation of the Latin expressions used in our proofs. The English definitions are based on the \href{https://dictionary.cambridge.org/}{Cambridge dictionary}.
\begin{table}[h]
    \centering
    \begin{tabular}{l|p{10cm}}
        Latin & English definition \\\hline 
               \textit{ad hoc} & made or happening only for a particular purpose or need, not planned before it happens \\
                \textit{a fortiori}  &for an even stronger reason than one that has already been accepted \\
                       \textit{in toto}  & as a total or whole \\
       \textit{ipso facto} & used to say that it is reasonable to state [...] something based on facts that are already known\\
       \textit{mutatis mutandis} & used when comparing two [...] things to say that although changes will be necessary in order to take account of different situations, the basic point remains the same.
    \end{tabular}
\end{table}

\bibliography{literature}

\end{document}